\documentclass{article}

\PassOptionsToPackage{numbers, compress}{natbib}

  \usepackage[preprint]{neurips_2026}

\usepackage{comment}
\usepackage[utf8]{inputenc} 
\usepackage[T1]{fontenc}   
\usepackage[colorlinks=true, allcolors=blue]{hyperref}
\usepackage{url}          
\usepackage{booktabs}       
\usepackage{amsfonts}     
\usepackage{nicefrac}      
\usepackage{microtype}     
\usepackage{xcolor}     
\usepackage{amsmath}
\usepackage{amsthm}
\usepackage{tikz}
\usepackage{nicefrac}
\usepackage{multirow}
\usepackage{multicol}
\usetikzlibrary{positioning}
\usetikzlibrary{calc}
\usepackage{enumitem}

\newtheorem{theorem}{Theorem}

\newtheorem{corollary}{Corollary}

\newtheorem*{theorem*}{Theorem}
\newtheorem*{lemma*}{Lemma}
\newtheorem*{proposition*}{Proposition}
\newtheorem*{corollary*}{Corollary}
\newtheorem*{example*}{Example}
\newtheorem{assumption}{Assumption}

\newcommand{\argmin}{\operatorname{argmin}}

\title{Feature Starvation as Geometric Instability\\ in Sparse Autoencoders}

\author{%
  Faris Chaudhry \\
  Imperial College London \\
  South Kensington Campus, London SW7 2AZ, UK\\
  \texttt{faris.chaudhry22@imperial.ac.uk} \\
  \And
  Keisuke Yano \\
  Institute for Statistical Mathematics \\
  10-3 Midori-cho, Tachikawa
Tokyo 190-8562, Japan \\
  \texttt{yano@ism.ac.jp} \\
  \AND
  Anthea Monod \\
  Imperial College London \\
  South Kensington Campus, London SW7 2AZ, UK \\
  \texttt{a.monod@imperial.ac.uk} \\
}

\begin{document}

\maketitle

\begin{abstract}
Sparse autoencoders (SAEs) are used to disentangle the dense, polysemantic internal representations of large language models (LLMs) into interpretable, monosemantic concepts. However, standard $\ell_1$-regularized SAEs suffer from feature starvation (dead neurons) and shrinkage bias, often requiring computationally expensive heuristic resampling and nondifferentiable hard-masking methods to bypass these challenges. We argue that feature starvation is not merely an empirical artifact of poor data diversity, but a fundamental optimization-geometric pathology of overcomplete dictionaries: the $\ell_1$-induced sparse coding map is unstable and fundamentally misaligned with shallow, amortized encoders. To address this structural instability, we introduce adaptive elastic net SAEs (AEN-SAEs), a fully differentiable architecture grounded in classical sparse regression. AEN-SAEs combine an $\ell_2$ structural term that enforces strong convexity and Lipschitz stability with adaptive $\ell_1$ reweighting that eliminates shrinkage bias and suppresses spurious features, thereby jointly controlling the curvature and interaction structure of the induced polyhedral geometry. Theoretically, we show that AEN-SAEs yield a Lipschitz-continuous sparse coding map and recover the global feature support under mild assumptions. Empirically, across synthetic settings and LLMs (Pythia 70M, Llama 3.1 8B), AEN-SAEs mitigate feature starvation without auxiliary heuristics while maintaining competitive reconstruction abilities.
\end{abstract}

\section{Introduction}

Large language models (LLMs) encode complex, polysemantic concepts within their dense residual streams, rendering their internal decision-making processes opaque. Recently, sparse autoencoders (SAEs) have emerged as a primary technique in mechanistic interpretability to linearly disentangle these representations into interpretable, monosemantic features~\citep{elhage2022toymodelssuperposition, bricken2023monosemanticity, cunningham2023sparseautoencoders}. By training an overcomplete dictionary with an $\ell_1$ sparsity penalty, SAEs project dense activations into a higher-dimensional space where individual basis directions correspond to human-understandable concepts.

The specific setup of an SAE intervention is illustrated in Figure~\ref{fig:sae_diagram}, building upon classical autoencoder frameworks for representation learning~\citep{hinton2006reducing}. Given a pretrained, frozen LLM with $F$ layers, input tokens $X_0$ are processed up to a chosen intermediate layer $M$ to extract a dense representation $X_M$. Middle layers are specifically targeted, as empirical evidence suggests this is where the most complex, polysemantic reasoning occurs, while earlier and later layers are heavily biased toward raw lexical processing and logit unembedding, respectively~\citep{geva2022transformers}. The SAE encoder projects $X_M$ into a sparse, high-dimensional latent space $h(X_M)$. The decoder then maps this projection back to the original dimension to produce the reconstructed activation $\hat{X}_M$. By substituting the original activations with these reconstructions $\hat{X}_M$ and measuring downstream recovery, individual basis directions corresponding to human-understandable concepts can be isolated and interpreted.

\begin{figure}[htbp]
    \centering
    \includegraphics[width=0.8\linewidth]{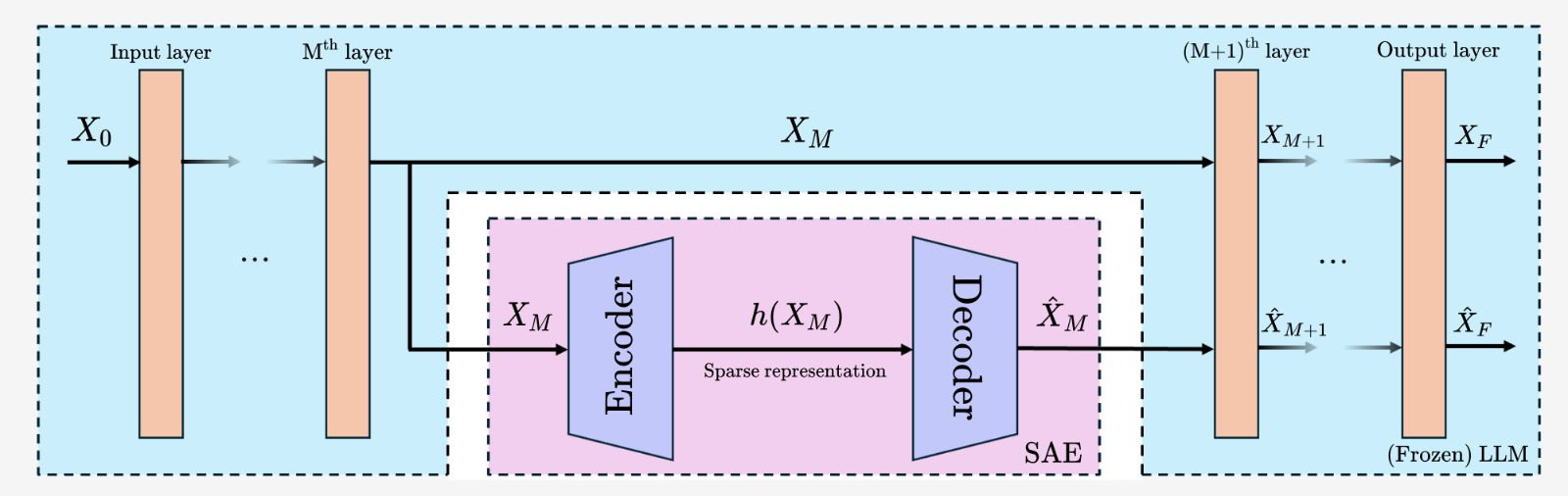}
    \caption{\textbf{Sparse autoencoder setup.} The SAE compresses intermediate LLM activations $X_M$ into a sparse representation $h(X_M)$ and reconstructs them as $\hat{X}_M$. To evaluate the dictionary's fidelity, the true activations are replaced with $\hat{X}_M$, and the forward pass is resumed through the frozen LLM to yield the reconstructed final logits $\hat{X}_F$. Performance is measured by the divergence between the original outputs $X_F$ and $\hat{X}_F$.}
    \label{fig:sae_diagram}
    \vspace{-5pt}
\end{figure}

Despite their empirical success, standard $\ell_1$-regularized SAEs suffer from severe optimization pathologies---most notably, \textit{feature starvation} (dead neurons) and \textit{magnitude shrinkage} (systematic underestimation of feature activations). While various architectural patches have been proposed to circumvent this (ranging from computationally expensive resampling heuristics to nondifferentiable hard-masking), these techniques treat the symptoms rather than addressing the underlying geometric root cause. In practice, such approaches rely heavily on auxiliary mechanisms---such as proxy gradients, resampling strategies, or additional loss terms---to prevent feature collapse. These interventions are often computationally expensive, introduce additional complexity, and can substantially increase the training cost required to reach competitive performance.

In this work, we show that feature starvation arises from the polyhedral geometry of the underlying sparse coding problem. Training an SAE is an attempt to approximate the computationally heavy least absolute shrinkage and selection operator (LASSO) objective \cite{tibshirani1996lasso} in a single, amortized forward pass. However, standard $\ell_1$ formulations induce an unstable and discontinuous sparse coding map, which is fundamentally misaligned with shallow, amortized encoders and leads to systematic training failure. To address these challenges, we introduce \emph{adaptive elastic net sparse autoencoders} (AEN-SAEs), a fully differentiable architecture that reduces the need for auxiliary heuristics while improving computational efficiency. Drawing on classical sparse regression, AEN-SAEs apply an adaptive $\ell_1$ penalty that vanishes for true signals to eliminate shrinkage bias. Additionally, we couple this with a constant $\ell_2$ structural anchor. This $\ell_2$ penalty ensures strong convexity and explicitly bounds the stability constant of the sparse coding map. Consequently, the AEN-SAE can be interpreted as a form of diagonal reweighting analogous to preconditioning, yielding a stable and fully differentiable architecture that naturally avoids dead features. Our main contributions are as follows:
\begin{itemize}[noitemsep]
    \item \textbf{Theoretical grounding:} We formalize the feature starvation pathology of standard SAEs through the lens of sparse recovery. We show that our AEN-SAE formulation satisfies oracle-like selection consistency while explicitly bounding the stability constant to guarantee polyhedral stability in the activation space.
    \item \textbf{Architectural efficiency:} We introduce a scalable, streaming mechanism for computing adaptive LASSO weights without requiring computationally expensive proxy gradients.
    \item \textbf{Empirical scaling:} Across controlled synthetic environments and real-world models (Pythia 70M~\citep{biderman2023pythia}, Llama 3.1 8B~\citep{grattafiori2024llama3}), we demonstrate that AEN-SAEs mitigate dead features and achieve competitive reconstruction curves compared to baselines.
\end{itemize}

\section{Background and Related Work}

We review the relevant background on SAEs and classical sparse regression and introduce the geometric perspective that underpins our approach. In particular, we connect SAE training to amortized sparse recovery and highlight the stability limitations of $\ell_1$-based formulations that motivate our method.

\subsection{Mechanistic Interpretability and SAEs}

The superposition hypothesis posits that neural networks encode more features than their ambient dimensionality by packing them into approximately orthogonal directions in activation space~\citep{elhage2022toymodelssuperposition}. SAEs aim to decode this structure by learning an overcomplete dictionary $D \in \mathbb{R}^{d_{\textit{model}} \times d_{\textit{dict}}}$ (with $d_{\textit{dict}} \gg d_{\textit{model}}$) that reconstructs the residual stream $x \in \mathbb{R}^{d_{\textit{model}}}$ as a sparse linear combination of feature directions.

The standard SAE objective takes the form of an $\ell_1$-regularized reconstruction loss:
\begin{equation*}
    \mathcal{L}_{\text{SAE}} = \mathbb{E}_x \left[ \frac{1}{2} \|x - (D h(x) + b_{\text{dec}})\|_2^2 + \lambda_1 \|h(x)\|_1 \right],
\end{equation*}
where $h(x) = \text{ReLU}(W_{\text{enc}}x + b_{\text{enc}})$ denotes the encoder's predicted sparse activation. In practice, optimizing this objective via gradient descent leads to two characteristic pathologies: \textit{shrinkage bias} and \textit{feature starvation}. The $\ell_1$ penalty systematically suppresses the magnitude of active features and prevents them from reaching their true magnitudes. Simultaneously, the encoder may fail to activate certain features altogether, trapping their preactivations in the negative regime. Since the derivative of the ReLU 
vanishes on negative inputs, these features receive no gradient signal and remain permanently inactive, effectively becoming \emph{dead neurons}. This phenomenon is known as \emph{feature starvation} and is the central failure mode we address in this work.

Early attempts to mitigate feature starvation relied on heuristic resampling of dead neurons~\citep{bricken2023monosemanticity}. To avoid full dictionary resets, subsequent work introduced \textit{Ghost Gradients}---an empirical heuristic that provides a proxy gradient signal to inactive features. Variants of this approach have become standard in large-scale SAE training, where auxiliary losses are used to artificially revive dead neurons~\citep{gao2024scaling}. While effective in practice, these methods are fundamentally heuristic: they introduce additional optimization pathways that are not derived from the underlying objective, require maintaining auxiliary gradient buffers, and artificially route loss, thereby increasing computational overhead.

Architectural innovations have sought to decouple the two failure modes of $\ell_1$. Gated SAEs separate feature selection from magnitude estimation, reducing shrinkage bias~\citep{rajamanoharan2024improving}. More recently, hard-masking approaches such as JumpReLU~\citep{rajamanoharan2024jumping, lieberum2024gemma}, TopK~\citep{gao2024scaling}, and BatchTopK~\citep{bussmann2024batchtopk} enforce a strictly $K$-sparse bottleneck and have emerged as state-of-the-art in efficiency. However, hard truncation exacerbates feature starvation, and these methods still rely heavily on auxiliary losses to prevent feature collapse~\citep{gao2024scaling}. Related approaches, such as feature choice SAEs~\citep{ayonrinde2024feature}, introduce additional routing mechanisms to achieve similar effects.

While empirically effective, hard-masking approaches introduce the fundamental theoretical limitation of nondifferentiability. Enforcing exact sparsity corresponds to solving an $\ell_0$-constrained problem, which is combinatorial and NP-hard~\citep{foucart2013compressivesensing}. Recent work~\citep{lee2025evaluating} has begun to incorporate ideas from sparse recovery theory into SAE design, showing that standard training can cause features to absorb or wedge together~\citep{chanin2025absorption}. These methods mitigate such effects by explicitly regularizing the dictionary to maintain quasi-orthogonality (low mutual coherence). Similarly, contemporary architectural alternatives have sought to resolve feature absorption through hierarchical constraints. Matryoshka SAEs \citep{bussmann2025matryoshka} force smaller nested dictionaries to learn independent, high-level concepts without relying on larger dictionaries, while subsequent extensions use attribution-guided distillation to freeze core features and reduce redundancy \citep{martinlinares2025attribution}. However, like previous methods, these approaches still rely heavily on nondifferentiable hard-masking mechanisms.

\subsection{Classical Polyhedral Sparse Regression}

The behavior of $\ell_1$-regularized SAEs is closely related to classical sparse recovery~\citep{foucart2013compressivesensing, wainwright2019highdimensional}. For a fixed dictionary $D$, the ideal sparse representation $h^{*}(x)$ is given by the \emph{LASSO problem}~\citep{tibshirani1996lasso}:
\begin{equation}
    \label{eq:lasso-sol}
    h^{*}(x) = \argmin_{h(x) \in \mathbb{R}^{d_{\text{dict}}}} \frac{1}{2} \|x - Dh(x)\|_2^2 + \lambda_1 \|h(x)\|_1.
\end{equation}
The goal is to recover the \emph{active set} $\mathcal{A} = \{i : h^{*}_i \neq 0\}$, whose stability is governed by geometric properties of the dictionary---in particular mutual incoherence and the conditioning of the active Gram matrix.

Classical recovery guarantees for LASSO rely on the \emph{irrepresentability condition} (IRC)~\citep{wainwright2019highdimensional, zhao2006consistency}, which requires that inactive features are not too correlated with the active set. In overcomplete LLM dictionaries, this condition is necessarily violated, leading the $\ell_1$ penalty to suppress true feature magnitudes in order to avoid selecting correlated noise. \emph{Adaptive LASSO}~\citep{zou2006adaptivelasso} addresses this suppression by reweighting the penalty, $w_i \propto |\hat{h}_i|^{-\gamma}$ with an initial estimate $\hat{h}$, achieving oracle support recovery where the active set is perfectly isolated and shrinkage bias is reduced. However, it requires a two-stage procedure---an initial estimate followed by reweighted optimization---which is incompatible with efficient amortized inference in LLMs. While online variants exist~\citep{mairal2010online}, they do not directly translate to the highly parallel, low-overhead setting required for LLMs.

A second factor related to stability is the curvature of the active solution, determined by the conditioning of the active Gram matrix $D_{\mathcal{A}}^\top D_{\mathcal{A}}$. When active features are highly collinear, this matrix becomes ill-conditioned, flattening the loss landscape and leading to unstable solutions. Elastic net regularization~\citep{zou2005elasticnet} addresses this problem by introducing an $\ell_2$ penalty ($\lambda_2 \|h\|_2^2$), which enforces strong convexity. However, it comes at the cost of increased shrinkage and feature grouping, which can degrade interpretability.

\subsection{Amortized LASSO and Polyhedral Geometry}
\label{subsec:amortized-lasso}

We interpret SAE training as learning an amortized approximation to the LASSO solution map: For a fixed dictionary $D$, the encoder $h(x)$ aims to predict the optimal sparse code $h^*(x)$ for each input $x$ in a single forward pass.

Standard approaches to amortized sparse recovery unroll iterative solvers into deep architectures such as LISTA~\citep{gregor2010lista, chen2021lista}. However, these deep architectures are computationally intractable at the scale of large language models. In contrast, SAEs rely on a strictly shallow architecture, typically a single linear projection followed by a ReLU---placing strong constraints on the complexity of the mappings they can represent. Our work addresses the resulting challenge of ensuring that the $\ell_1$-induced solution map is sufficiently stable to be learned by such shallow encoders.

A key distinction between classical sparse recovery and the SAE setting has to do with the role of sparsity. In classical compressed sensing, the sparsity level is an unknown, fixed property of the data-generating process. In contrast, in mechanistic interpretability, sparsity is typically a design choice, selected to balance interpretability and reconstruction.

The ability to learn a shallow amortized map depends critically on the stability of the target solution. Under an $\ell_1$ penalty, the solution space is polyhedral, and the sensitivity of the optimal sparse codes $h^*(x)$ to perturbations in $x$ is governed by a local error bound characterized by the \textit{Hoffman constant} $H$~\citep{hoffman1952approximate, robinson1981polyhedral}. This constant depends on two geometric properties of the dictionary: the conditioning of the active Gram matrix, via $\lambda_{\min}(D_{\mathcal{A}}^\top D_{\mathcal{A}})$, and the interaction between active and inactive features, measured by $\|D_{\mathcal{A}^c}^\top D_{\mathcal{A}}\|_2$. For a given active set, it scales as
\begin{equation*}
    H_{\mathcal{A}} \sim \mathcal{O}\left( \|D_{\mathcal{A}^c}^\top D_{\mathcal{A}}\|_2 \cdot \frac{1}{\lambda_{\min}(D_{\mathcal{A}}^\top D_{\mathcal{A}})} \right).
\end{equation*}
As the dictionary becomes more coherent, $H$ diverges, and the solution map becomes increasingly nonsmooth. This instability makes the mapping difficult for shallow encoders to approximate.

\section{Adaptive Elastic Net SAEs (AEN-SAEs)}
\label{sec:aen}

To address the geometric instability of $\ell_1$-based sparse coding, we introduce \emph{adaptive elastic net sparse autoencoders} (AEN-SAEs), which combine adaptive sparsity with explicit curvature control. The key idea is to stabilize the amortized solution map by simultaneously suppressing spurious feature interactions and enforcing strong convexity.

Let $h(x) \in \mathbb{R}^{d_{\text{dict}}}$ denote the sparse latent representation predicted by the encoder for an input $x \in \mathbb{R}^{d_{\text{model}}}$. We define the AEN-SAE objective as:
\begin{equation}
    \label{eq:aen-sae-loss}
    \mathcal{L}_{AEN} = \mathbb{E}_x \left[ \frac{1}{2} \|x - (D h(x) + b_{\text{dec}})\|_2^2 + \lambda_1 \sum_{i=1}^{d_{\text{dict}}} w_i |h_i(x)| + \lambda_2 \|h(x)\|_2^2 \right].
\end{equation}
Here, $w_i$ denotes a feature-specific penalty weight for the $i$-th feature, and $\lambda_1, \lambda_2 > 0$ control the sparsity and structural regularization, respectively. Compared to the standard SAE objective, this replaces the uniform $\ell_1$ penalty with an adaptive weighted $\ell_1$ penalty to isolate the active set, and introduces a fixed $\ell_2$ structural anchor to bound active curvature and ensure stability through strong convexity.

A direct implementation of adaptive LASSO requires a two-pass algorithm, which is computationally infeasible for streaming LLM tokens. Instead, we implement a lightweight online approximation based on an exponential moving average (EMA) of feature activations. Let $h^{(t)}$ denote the post-ReLU sparse activations at training step $t$ across a batch of size $B$. We update the EMA synchronously:
\begin{equation*}
   \bar{h}^{(t)} = \beta \bar{h}^{(t-1)} + (1 - \beta) \frac{1}{B} \sum_{b=1}^{B} |h_b^{(t)}|, 
\end{equation*}
where $\beta$ is the momentum factor.

To ensure robustness across layers and models, we employ a scale-invariant adaptive weighting mechanism. Let $p \in (0,1]$ denote a top-percentile hyperparameter. At step $t$, let $\mathcal{S}_p^{(t)}$ be the index set of the top $p \cdot d_{dict}$ features with the largest EMA activations. We define a reference activation as the mean activity over this set. The raw adaptive weight for feature $i$ is then calculated asynchronously relative to this reference:
\begin{equation*}
   \bar{h}_{\text{ref}}^{(t)} = \frac{1}{|\mathcal{S}_p^{(t)}|} \sum_{j \in \mathcal{S}_p^{(t)}} \bar{h}_j^{(t)} \quad \text{and} \quad w_i^{(t)} = \mathrm{clip}\!\left(\left(\frac{\bar{h}_{\mathrm{ref}}^{(t)}}{\bar{h}_i^{(t)}+\epsilon}\right)^\gamma, \, w_{\min},\, w_{\max}\right). 
\end{equation*}
where $\epsilon = 10^{-5}$ ensures numerical stability, and the bounds $[w_{\min}, w_{\max}]$ prevent extreme gradient scaling. This normalizes feature importance relative to a high-activity reference set and ensures consistent scaling across layers.

Adaptive weighting introduces a concentration dynamic: if a feature remains inactive early in training, its EMA decays toward zero, causing its penalty weight to saturate at $w_{\max}$ and effectively eliminating it from the model. To prevent premature feature death during this initialization phase, we introduce a delayed linear warmup schedule $\rho(t) \in [0,1]$:
\begin{equation*}
    \rho(t)= \begin{cases} 0, & t < T_{\mathrm{warmup}},\\[3pt] \min\!\left(1,\dfrac{t-T_{\mathrm{warmup}}}{T_{\mathrm{ramp}}}\right), & t \ge T_{\mathrm{warmup}}. \end{cases}
\end{equation*}
The effective penalty weight is interpolated as $w_{i, \mathrm{eff}}^{(t)} = 1.0 + \rho(t) (w_i^{(t)} - 1.0)$, so that training begins with a uniform $\ell_1$ penalty and transitions to the adaptive regime as features become established.

The $\ell_2$ term controls the conditioning of the active Gram matrix, while the adaptive weighting mechanism suppresses interaction leakage between active and inactive features. Inactive features receive large penalties ($w_i \to w_{\max}$), driving their activations toward zero and reducing the interaction term $\|D_{\mathcal{A}^c}^\top D_{\mathcal{A}}\|_2$. Conversely, features that are consistently active are assigned small weights ($w_i \to w_{\min}$), effectively removing the $\ell_1$ penalty on the active set and avoiding shrinkage bias.

Theoretical results for the AEN-SAE architecture are provided in Appendix~\ref{sec:theory}. In particular, under a timescale separation assumption, we show that the proposed weighting mechanism achieves oracle-like support recovery in the idealized setting, identifying the active set and reducing shrinkage bias.

\section{Experiments}

We evaluate AEN-SAEs on synthetic and real-world settings to test their robustness under dictionary coherence. We begin with a controlled spiked-model experiment to isolate the effect of coherence, then evaluate on Pythia 70M and Llama 3.1 8B to assess performance on real LLM activations and hyperparameter transfer across model sizes. Full experimental and implementation details are provided in Appendix~\ref{app:setup}.

For LLM-based experiments, we stream text from a deduplicated version of the Pile~\citep{gao2020pile800gbdatasetdiverse}, extract residual-stream activations from a frozen model at a chosen layer, and form token-level batches. Activations are $\ell_2$-normalized per token to ensure consistent scaling across models.

SAEs are trained with Adam at a fixed learning rate using reconstruction loss with regularization. Gradients are clipped for stability, and decoder columns are renormalized after each step. Adaptive-weight models update EMA statistics online and apply reweighting after a warmup period. No auxiliary resampling or proxy-gradient methods are used for any model.

We evaluate reconstruction (MSE, explained variance), sparsity ($\ell_0$), and feature starvation (dead neuron rate). To probe geometric structure, we also measure dictionary coherence, condition numbers, and feature-utilization statistics (e.g., entropy and Gini coefficient). For LLMs, we additionally assess downstream fidelity via reconstruction patching. Full metric definitions are provided in Appendix~\ref{sec:extended-methodology}.

\subsection{Spiked Model: Coherence-Induced Instability}

To isolate the effect of dictionary coherence, we consider a synthetic spiked model where each atom is constructed as a mixture of an independent direction and a shared component~\citep{wainwright2019highdimensional}:
\begin{equation*}
    d_j = \frac{\sqrt{1-\rho}u_j + \sqrt{\rho}v}{\|\sqrt{1-\rho}u_j + \sqrt{\rho}v\|_2},
\end{equation*}
where $\rho \in [0,1]$ controls the expected pairwise coherence. This allows for the variation of coherence while keeping sparsity and noise fixed. A detailed description of the spiked model construction and additional experimental results are provided in Appendix~\ref{app:spiked}.

We generate data by sampling $k$-sparse codes ($k=16$) and forming $x^* = D^* h^*$ using a fixed teacher dictionary ($d_{\text{model}}=256$, $d_{\text{dict}}=1024$). Full experimental details are provided in Appendix~\ref{app:setup}.

\textbf{Low coherence ($\rho=0$).} When dictionary atoms are independent, all methods perform well (Figure~\ref{fig:spiked-pareto}). In this regime, feature selection is trivial and even hard-masking methods exhibit low failure rates. AEN-SAE maintains full dictionary utilization with no dead features, at the cost of a small reduction in explained variance due to the $\ell_2$ regularization.

\begin{figure}
    \centering
    \includegraphics[width=0.8\linewidth]{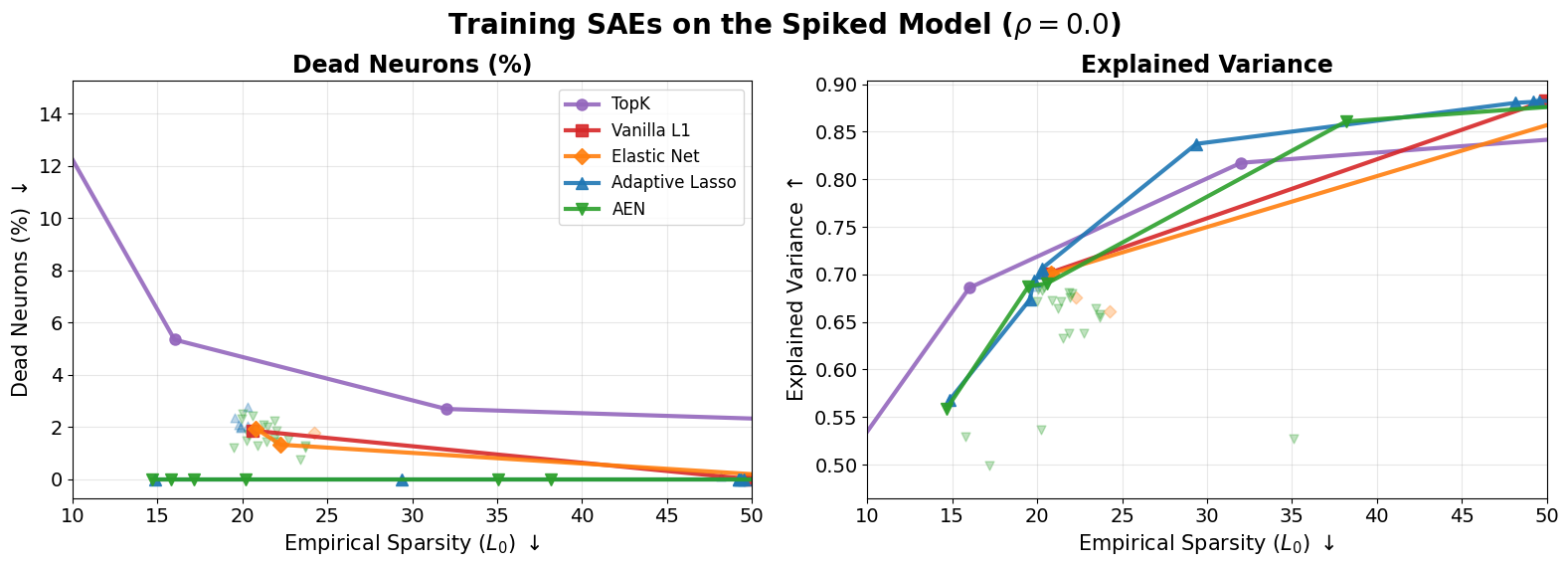}
    \caption{\textbf{Pareto curves for spiked dictionary ($\rho = 0$).} When dictionary atoms are uncorrelated, all architectures work reasonably well, with similar explained variance curves (right). However, even in this simple case, dead neurons in TopK exist while AEN-SAE solves this. Note that adaptive LASSO concentrates around the true sparsity level and thus there is no Pareto curve; it is less flexible for mechanistic interpretability since the sparsity is an interpretability choice rather than ground truth.}
    \label{fig:spiked-pareto}
    \vspace{-10pt}
\end{figure}

\textbf{High coherence ($\rho=0.9$).} In contrast, high coherence induces severe instability across standard SAE architectures. Continuous relaxations become highly sensitive to hyperparameters, while adaptive LASSO frequently collapses to degenerate solutions. Hard-masking methods such as TopK suffer from severe feature starvation, eliminating over 95\% of the dictionary at typical sparsity levels (Table~\ref{tab:synthetic_rho0.9_comparison}).

AEN-SAE mitigates this failure mode by stabilizing both curvature and feature interactions. While it does not completely eliminate feature starvation, it reduces dead features to $\approx 41\%$ and maintains substantially higher reconstruction quality. Moreover, AEN yields significantly improved geometric structure, reducing the active condition number ($\kappa \approx 21$ vs.\ $769$ for elastic net) and lowering interaction leakage.

These results demonstrate that coherence fundamentally limits the stability of $\ell_1$-based sparse coding, and that AEN provides a principled mechanism for mitigating this instability.

\begin{table}[htbp]
\centering
\caption{\textbf{Synthetic spiked model comparison ($\rho=0.9$).} Head-to-head performance of TopK vs. AEN-SAE in a highly collinear regime at matched empirical sparsity levels.}
\label{tab:synthetic_rho0.9_comparison}
\resizebox{0.6\textwidth}{!}{%
\begin{tabular}{llrrr}
\toprule
\textbf{Empirical $\ell_0$} & \textbf{Architecture} & \textbf{Dead Neurons $\downarrow$} & \textbf{Expl. Var. $\uparrow$} & \textbf{Shrinkage $\uparrow$} \\
\midrule
\multirow{2}{*}{$\approx 16$} & TopK & 97.7\% & 0.470 & 0.996 \\
& \textbf{AEN-SAE} & \textbf{40.9\%} & \textbf{0.507} & 0.996 \\
\cmidrule{1-5}
\multirow{2}{*}{$\approx 32$} & TopK & 96.5\% & 0.547 & 0.996 \\
& \textbf{AEN-SAE} & \textbf{41.0\%} & \textbf{0.665} & \textbf{0.997} \\
\bottomrule
\end{tabular}%
}
\vspace{-10pt}
\end{table}

\subsection{Pythia 70M}

We evaluate AEN-SAEs on activations from a frozen Pythia 70M model~\citep{biderman2023pythia}, training an $8\times$ overcomplete dictionary ($d_{\text{dict}}=4096$) on layer-3 residuals. Full experimental details and extended analysis for Pythia 70M are provided in Appendix~\ref{app:pythia}.

\textbf{Hyperparameters and efficiency.}
Although AEN-SAEs introduce additional hyperparameters, we find them stable across runs and consistent with classical adaptive elastic net guidance. In practice, only the base sparsity penalty $\lambda_1$ needs to be tuned to reach a desired sparsity level. The method introduces negligible overhead ($<0.1\%$ FLOPs increase compared to TopK) and requires no auxiliary resampling or proxy-gradient mechanisms. A detailed hyperparameter sweep and robustness analysis are provided in Appendix~\ref{sec:hyperparam-sweep-pythia}.

\textbf{Dictionary coherence as a failure modality.} Unlike synthetic settings, LLM activations exhibit strong feature correlations, inducing a highly coherent dictionary. In this regime, standard SAE variants fail in distinct ways. Vanilla $\ell_1$ requires large penalties to suppress correlated features, leading to severe shrinkage (up to $\approx 25\%$ magnitude loss at $\ell_0 \approx 22$) and high sensitivity to hyperparameters. Elastic net stabilizes training but produces dense, poorly conditioned representations (median $\kappa > 10^{13}$). Adaptive LASSO lacks a structural $\ell_2$ anchor and becomes numerically unstable under collinearity and frequently collapses, which is consistent with known limitations of $\ell_1$ methods in correlated settings.

\textbf{Reconstruction vs.\ feature utilization.} Figure~\ref{fig:pythia70m-pareto} and Table~\ref{tab:pythia_aen_vs_topk} illustrate a clear trade-off between raw reconstruction and feature utilization. At $\ell_0 \approx 32$, TopK achieves higher explained variance ($0.76$), but does so by concentrating capacity into highly redundant feature hubs (max coherence $0.99$). Despite this high extreme, TopK exhibits a low p90 coherence (0.327), suggesting that it learns a sparse set of nearly identical high-frequency features while starving the surrounding correlated semantic space (nearly 14\% dead features).

In contrast, AEN distributes activation mass across correlated feature groups, which substantially reduces feature starvation (down to $\approx 1\%$ at $\ell_0 \approx 32$) and lowers extreme redundancy (max coherence $0.96$, p90 coherence $0.516$). This results in a more balanced and geometrically structured dictionary, though with a modest reduction in explained variance. AEN operates at a different point in the trade-off space, prioritizing feature utilization and geometric stability over maximal reconstruction.

\begin{figure}
    \centering
    \includegraphics[width=0.8\linewidth]{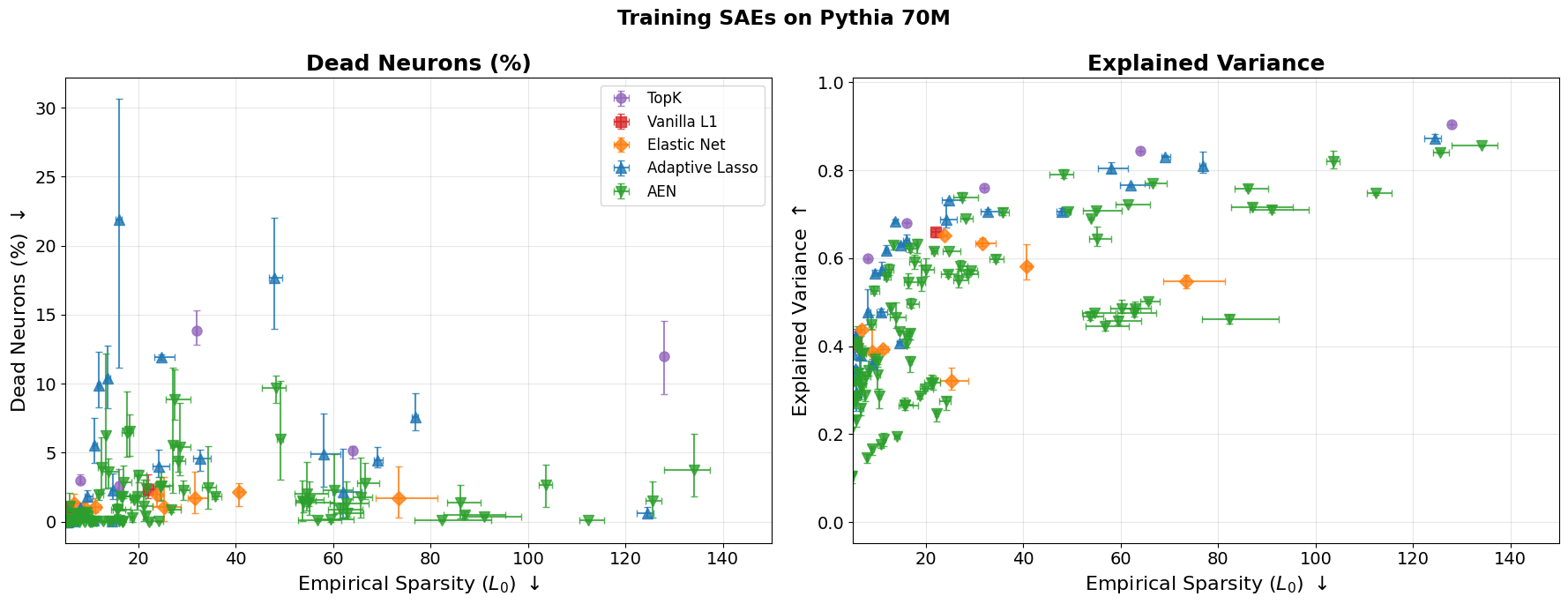}
    \caption{\textbf{Pareto frontiers on the Pythia 70M residual stream across three random seeds for all tested configurations.} The left panel illustrates feature starvation. The right panel displays raw reconstruction fidelity. Error bars denote $[\min, \max]$ bounds.}
    \label{fig:pythia70m-pareto}
    \vspace{-10pt}
\end{figure}

\begin{table}[htbp]
\centering
\caption{\textbf{Head-to-head performance on Pythia 70M.} Comparing AEN to the TopK baseline at targeted sparsity bottlenecks across three random seeds. Results are reported as mean [min, max].}
\label{tab:pythia_aen_vs_topk}
\resizebox{0.9\textwidth}{!}{%
\begin{tabular}{llrrrrr}
\toprule
\textbf{Target $\ell_0$} & \textbf{Architecture} & \textbf{Achieved $\ell_0$} & \textbf{Dead Neurons. $\downarrow$} & \textbf{Expl. Var. $\uparrow$} & \textbf{p90 Coh. $\uparrow$} & \textbf{Max Coh. $\downarrow$} \\
\midrule
\multirow{2}{*}{$\approx 16$} & TopK & 16.0 [16.0, 16.0] & \textbf{2.3\%} [2.0, 2.7] & \textbf{0.680} [0.679, 0.681] & 0.326 [0.319, 0.334] & 0.986 [0.986, 0.987] \\
& \textbf{AEN-SAE} & 18.2 [18.0, 18.4] & 6.6\% [4.8, 7.8] & 0.632 [0.612, 0.644] & \textbf{0.685} [0.677, 0.696] & \textbf{0.936} [0.935, 0.938] \\
\midrule
\multirow{2}{*}{$\approx 32$} & TopK & 32.0 [32.0, 32.0] & 13.9\% [12.8, 15.3] & \textbf{0.760} [0.759, 0.760] & 0.329 [0.327, 0.330] & 0.989 [0.989, 0.990] \\
& \textbf{AEN-SAE} & 33.9 [33.7, 34.1] & \textbf{1.3\%} [1.3, 1.4] & 0.719 [0.696, 0.759] & \textbf{0.628} [0.603, 0.665] & \textbf{0.963} [0.957, 0.968] \\
\bottomrule
\end{tabular}%
}
\vspace{-5pt}
\end{table}

\textbf{Feature utilization at scale.} At higher sparsity levels ($\ell_0 \approx 128$), the difference becomes more pronounced (Table~\ref{tab:feature_utilization}). TopK concentrates $86\%$ of activation mass into the top $10\%$ of features (Gini $0.903$), indicating severe underutilization of capacity. AEN mitigates this effect, reducing the Gini coefficient to $0.670$ and increasing entropy, while lowering dead features to $\approx 3\%$. This behavior is driven by the adaptive reweighting mechanism: highly active features are penalized more strongly, which prevents dominance of redundant hubs and encourages recruitment of underutilized features.

\begin{table}[htbp]
\centering
\caption{\textbf{Feature utilization and starvation at relaxed sparsities ($\ell_0 \approx 128$).} Results are reported as mean [min, max] across random seeds. While TopK maintains higher explained variance, it does so by permanently starving $\approx 12\%$ of the dictionary and concentrating over $86\%$ of its activation mass into a narrow subset of features. AEN resolution of this concentration is statistically robust across seeds, yielding higher entropy, a vastly lower Gini coefficient, and a fully utilized dictionary.}
\label{tab:feature_utilization}
\resizebox{\textwidth}{!}{%
\begin{tabular}{lrrrrrrr}
\toprule
\textbf{Architecture} & \textbf{Achieved $\ell_0$} & \textbf{Expl. Var. $\uparrow$} & \textbf{Shrinkage $\uparrow$} & \textbf{Dead Neurons. $\downarrow$} & \textbf{Gini $\downarrow$} & \textbf{Norm. Entropy $\uparrow$} & \textbf{Top 10\% Mass $\downarrow$} \\
\midrule
TopK & 128.0 [128.0, 128.0] & \textbf{0.904} [0.904, 0.905] & \textbf{0.954} [0.953, 0.955] & 12.0\% [9.3, 14.5] & 0.903 [0.902, 0.904] & 0.755 [0.754, 0.757] & 86.1\% [85.8, 86.4] \\
\textbf{AEN-SAE} & 127.5 [127.3, 128.0] & 0.854 [0.853, 0.856] & 0.876 [0.868, 0.881] & \textbf{3.1\%} [1.8, 5.2] & \textbf{0.670} [0.641, 0.685] & \textbf{0.899} [0.893, 0.909] & \textbf{48.2\%} [44.6, 50.1] \\
\bottomrule
\end{tabular}%
}
\vspace{-8pt}
\end{table}

Overall, these results demonstrate that while hard-masking methods maximize reconstruction, they do so by collapsing onto a highly redundant subset of features. AEN instead provides a principled mechanism for improving feature utilization and geometric stability under coherence, without relying on auxiliary heuristics.

\subsection{Llama 3.1 8B}
\label{subsec:llama}

To evaluate AEN-SAEs at scale, we extract dense representations from the middle layer (layer 16) of a frozen Llama 3.1 8B~\citep{grattafiori2024llama3} ($d_{\text{model}} = 4096$) and train a $32\times$ overcomplete dictionary ($d_{\text{dict}}=131,072$). We stream the deduplicated Pile dataset with a sequence length of 128, applying per-token $\ell_2$ normalization. Optimization proceeds via Adam (batch size 4,096, learning rate $5 \times 10^{-4}$) for 100,000 steps ($\approx 400$M tokens), utilizing a 4,000-step linear warmup for the AEN-SAE adaptive weights. The SAE is trained in full precision (FP32) against BF16 LLM representations, and final downstream validation is performed on a held-out shard of 20,000 documents.

\textbf{Hyperparameter transfer and scale invariance.} Exhaustive hyperparameter tuning at this scale is computationally prohibitive. However, AEN-SAEs exhibit strong scale invariance; specifically, structural hyperparameters transfer directly from smaller models. This approach is philosophically aligned with maximal update parameterization ($\mu P$) \citep{yang2022tensorprogramsvtuning}, which establishes scaling laws that keep optimal hyperparameters stable as model width increases. We thus fix the structural hyperparameters $(\lambda_2 = 10^{-4}, \gamma = 0.5, \beta = 0.9999, p = 0.05, w_{\min} = 0.01, w_{\max} = 10.0)$ based on Pythia 70M and tune only $\lambda_1$ to target sparsity and hit target $\ell_0$ bottlenecks. Empirically, this transfer is successful and enables reliable deployment without large-scale hyperparameter sweeps.

\textbf{Failure of hard-masking at scale.} Table~\ref{tab:llama_aen_vs_topk} reveals the geometric cost of hard-masking for large-scale models. Across all sparsity levels, TopK exhibits extreme feature starvation, with $\approx 80\%$ of the dictionary remaining inactive. At the same time, it achieves high reconstruction by collapsing onto a highly redundant subset of features, evidenced by maximum coherence reaching $1.000$ immediately at $\ell_0 \approx 32$ and maintaining a heavily skewed p90 coherence across all capacities. This behavior reflects a degenerate solution: rather than exploring the available 131,072-dimensional capacity, the model concentrates activation mass into a small set of nearly identical features, leaving most of the dictionary unused.

\textbf{AEN: utilization under coherence.} In contrast, AEN-SAEs explicitly suppress the redundancy exhibited by TopK. By penalizing frequently co-activated features, AEN redistributes activation mass across correlated groups, reducing both feature starvation and extreme coherence. For example, at $\ell_0 \approx 128$, the p90 coherence drops from $0.438$ (TopK) to $0.218$, while dead features are reduced from $\approx 80\%$ to $\approx 39\%$. Notably, even AEN retains a nontrivial fraction of inactive features, suggesting that complete elimination of feature starvation may be fundamentally constrained and not purely an optimization issue, but also a consequence of sampling limitations at this scale.

While AEN incurs a reduction in explained variance, this reflects a different operating point in the trade-off space: AEN prioritizes feature utilization and geometric stability over maximal reconstruction.

\textbf{Scaling behavior.} Our results confirm that the benefits of AEN persist at large scale. The same geometric mechanism observed in synthetic and mid-scale experiments-coherence-induced collapse and its mitigation via adaptive reweighting—remains the dominant factor governing performance. Crucially, this is achieved without auxiliary heuristics and with hyperparameters transferred directly from smaller models. Overall, AEN-SAEs provide a scalable approach to sparse representation learning that prioritizes feature utilization and geometric stability in regimes where standard methods degenerate. Notably, this performance is achieved with negligible computational overhead relative to TopK and without auxiliary resampling or proxy-gradient mechanisms.

\begin{table}[htbp]
\centering
\caption{\textbf{Head-to-head performance on Llama 3.1 8B.} Comparing AEN to TopK baseline at targeted sparsity bottlenecks.}
\label{tab:llama_aen_vs_topk}
\resizebox{0.9\textwidth}{!}{
\begin{tabular}{llrrrrrrr}
\toprule
\textbf{Target $L_0$} & \textbf{Architecture} & \textbf{Achieved $L_0$} & \textbf{Dead Feat. $\downarrow$} & \textbf{Expl. Var. $\uparrow$} & \textbf{MSE $\downarrow$} & \textbf{p50 Coh. $\downarrow$} & \textbf{p90 Coh. $\downarrow$} & \textbf{Max Coh. $\downarrow$} \\
\midrule
\multirow{2}{*}{$\approx 32$} & Top-$K$ & 32.0 & 88.0\% & \textbf{0.505} & \textbf{0.406} & 0.127 & 0.313 & 1.000 \\
& \textbf{AEN-SAE} & 34.5 & \textbf{55.4\%} & 0.199 & 0.671 & \textbf{0.097} & \textbf{0.143} & \textbf{0.585} \\
\cmidrule{1-9}
\multirow{2}{*}{$\approx 64$} & Top-$K$ & 64.0 & 80.2\% & \textbf{0.580} & \textbf{0.342} & 0.097 & 0.360 & 1.000 \\
& \textbf{AEN-SAE} & 68.8 & \textbf{45.0\%} & 0.400 & 0.503 & \textbf{0.093} & \textbf{0.182} & \textbf{0.862} \\
\cmidrule{1-9}
\multirow{2}{*}{$\approx 128$} & Top-$K$ & 127.8 & 80.3\% & \textbf{0.656} & \textbf{0.294} & 0.100 & 0.438 & 1.000 \\
& \textbf{AEN-SAE} & 125.8 & \textbf{39.3\%} & 0.469 & 0.443 & \textbf{0.087} & \textbf{0.218} & \textbf{0.989} \\
\cmidrule{1-9}
\end{tabular}
}
\end{table}

\section{Discussion, Limitations, and Future Research}

In this work, we reframed feature starvation and shrinkage bias in sparse autoencoders as geometric and optimization pathologies arising from overcomplete, highly coherent dictionaries. Building on classical results from high-dimensional statistics, we introduced AEN-SAEs: a fully differentiable architecture that combines online adaptive $\ell_1$ reweighting with an $\ell_2$ structural term to stabilize the sparse coding map.

Our analysis shows that coherence is a primary driver of failure in standard SAE formulations. In such regimes, $\ell_1$-based methods become unstable, while hard-masking approaches achieve strong reconstruction by collapsing onto a small, highly redundant subset of features. AEN-SAEs provide a principled alternative: by directly controlling curvature and feature interactions, they improve feature utilization and geometric stability without relying on auxiliary heuristics such as resampling or proxy gradients.

Beyond architectural design, our results highlight the importance of evaluating sparse representations through geometric diagnostics, including dictionary coherence, interaction between active and inactive sets, and the conditioning of the active manifold. These metrics provide a more faithful view of representation structure than reconstruction alone, and are particularly relevant for interpretability, where underutilized dictionaries may fail to capture rare but semantically important features.

\textbf{Limitations and future work.}  While AEN-SAEs improve stability and utilization, several limitations remain. First, as with all $\ell_1$-based methods, achieving a target sparsity level requires tuning $\lambda_1$, which is less direct than specifying a fixed $k$ in hard-masking approaches. Developing adaptive or self-tuning sparsity mechanisms remains an important direction.

Second, although AEN significantly reduces feature starvation, a nontrivial fraction of neurons remain inactive at scale (e.g., $40$--$50\%$ at the 8B model). This suggests that feature inactivity is not purely an optimization failure but also reflects sampling limitations at scale: finite training horizons may be insufficient to activate highly specialized features in very large dictionaries.

Third, our evaluation focuses on geometric and statistical properties of the learned representations. A comprehensive assessment of human-centered interpretability~\citep{liu2025olmotracetracinglanguagemodel} remains for future work, and will likely require both larger-scale training and dedicated evaluation protocols.

Finally, our results suggest a promising hybrid direction: combining geometry-aware methods such as AEN with structure-informed resampling strategies based on coherence or feature utilization, or integrating adaptive $\ell_2$ anchors into hierarchical architectures such as Matryoshka SAEs~\citep{bussmann2025matryoshka} to enforce both geometric stability and multi-level feature abstraction.

\begin{ack}
A.M.~is supported by the EPSRC AI Hub on Mathematical Foundations of Intelligence: An ``Erlangen Programme'' for AI [EP/Y028872/1]. 
K.Y.~is supported by JSPS KAKENHI (24K15120, 24H00247, 26K02871).
\end{ack}

\bibliographystyle{abbrvnat}
\bibliography{bibliography}

@misc{bricken2023monosemanticity,
  title        = {Towards Monosemanticity: Decomposing Language Models With Dictionary Learning},
  author       = {Bricken, Trenton and Templeton, Adly and Batson, Joshua and Chen, Brian and Jermyn, Adam and Conerly, Tom and Turner, Nicholas L. and Anil, Cem and Denison, Carson and Askell, Amanda and Lasenby, Robert and Wu, Yifan and Kravec, Shauna and Schiefer, Nicholas and Maxwell, Tim and Joseph, Nicholas and Tamkin, Alex and Nguyen, Karina and McLean, Brayden and Burke, Josiah E. and Hume, Tristan and Carter, Shan and Henighan, Tom and Olah, Chris},
  year         = {2023},
  url = {https://transformer-circuits.pub/2023/monosemantic-features},
}

@misc{cunningham2023sparseautoencoders,
      title={Sparse Autoencoders Find Highly Interpretable Features in Language Models}, 
      author={Hoagy Cunningham and Aidan Ewart and Logan Riggs and Robert Huben and Lee Sharkey},
      year={2023},
      eprint={2309.08600},
      archivePrefix={arXiv},
      primaryClass={cs.LG},
      url={https://doi.org/10.48550/arXiv.2309.08600}, 
}

@misc{elhage2022toymodelssuperposition,
      title={Toy Models of Superposition}, 
      author={Nelson Elhage and Tristan Hume and Catherine Olsson and Nicholas Schiefer and Tom Henighan and Shauna Kravec and Zac Hatfield-Dodds and Robert Lasenby and Dawn Drain and Carol Chen and Roger Grosse and Sam McCandlish and Jared Kaplan and Dario Amodei and Martin Wattenberg and Christopher Olah},
      year={2022},
      eprint={2209.10652},
      archivePrefix={arXiv},
      primaryClass={cs.LG},
      url={https://doi.org/10.48550/arXiv.2209.10652}, 
}

@article{hinton2006reducing,
  title={Reducing the dimensionality of data with neural networks},
  author={Hinton, Geoffrey E and Salakhutdinov, Ruslan R},
  journal={Science},
  volume={313},
  number={5786},
  pages={504--507},
  year={2006},
  publisher={American Association for the Advancement of Science},
  url={https://doi.org/10.1126/science.1127647}
}

@misc{geva2022transformers,
      title={Transformer Feed-Forward Layers Build Predictions by Promoting Concepts in the Vocabulary Space}, 
      author={Mor Geva and Avi Caciularu and Kevin Ro Wang and Yoav Goldberg},
      year={2022},
      eprint={2203.14680},
      archivePrefix={arXiv},
      primaryClass={cs.CL},
      url={https://doi.org/10.48550/arXiv.2203.14680}, 
}

@misc{gao2024scaling,
      title={Scaling and evaluating sparse autoencoders}, 
      author={Leo Gao and Tom Dupré la Tour and Henk Tillman and Gabriel Goh and Rajan Troll and Alec Radford and Ilya Sutskever and Jan Leike and Jeffrey Wu},
      year={2024},
      eprint={2406.04093},
      archivePrefix={arXiv},
      primaryClass={cs.LG},
      url={https://doi.org/10.48550/arXiv.2406.04093}, 
}

@misc{rajamanoharan2024improving,
      title={Improving Dictionary Learning with Gated Sparse Autoencoders}, 
      author={Senthooran Rajamanoharan and Arthur Conmy and Lewis Smith and Tom Lieberum and Vikrant Varma and János Kramár and Rohin Shah and Neel Nanda},
      year={2024},
      eprint={2404.16014},
      archivePrefix={arXiv},
      primaryClass={cs.LG},
      url={https://doi.org/10.48550/arXiv.2404.16014}, 
}

@misc{grattafiori2024llama3,
      title={The {Llama} 3 Herd of Models}, 
      author={Aaron Grattafiori and Abhimanyu Dubey and Abhinav Jauhri and Abhinav Pandey and Abhishek Kadian and Ahmad Al-Dahle and Aiesha Letman and Akhil Mathur and Alan Schelten and Alex Vaughan and Amy Yang and Angela Fan and Anirudh Goyal and Anthony Hartshorn and Aobo Yang and Archi Mitra and Archie Sravankumar and Artem Korenev and Arthur Hinsvark and Arun Rao and Aston Zhang and Aurelien Rodriguez and Austen Gregerson and Ava Spataru and Baptiste Roziere and Bethany Biron and Binh Tang and Bobbie Chern and Charlotte Caucheteux and Chaya Nayak and Chloe Bi and Chris Marra and Chris McConnell and Christian Keller and Christophe Touret and Chunyang Wu and Corinne Wong and Cristian Canton Ferrer and Cyrus Nikolaidis and Damien Allonsius and Daniel Song and Danielle Pintz and Danny Livshits and Danny Wyatt and David Esiobu and Dhruv Choudhary and Dhruv Mahajan and Diego Garcia-Olano and Diego Perino and Dieuwke Hupkes and Egor Lakomkin and Ehab AlBadawy and Elina Lobanova and Emily Dinan and Eric Michael Smith and Filip Radenovic and Francisco Guzmán and Frank Zhang and Gabriel Synnaeve and Gabrielle Lee and Georgia Lewis Anderson and Govind Thattai and Graeme Nail and Gregoire Mialon and Guan Pang and Guillem Cucurell and Hailey Nguyen and Hannah Korevaar and Hu Xu and Hugo Touvron and Iliyan Zarov and Imanol Arrieta Ibarra and Isabel Kloumann and Ishan Misra and Ivan Evtimov and Jack Zhang and Jade Copet and Jaewon Lee and Jan Geffert and Jana Vranes and Jason Park and Jay Mahadeokar and Jeet Shah and Jelmer van der Linde and Jennifer Billock and Jenny Hong and Jenya Lee and Jeremy Fu and Jianfeng Chi and Jianyu Huang and Jiawen Liu and Jie Wang and Jiecao Yu and Joanna Bitton and Joe Spisak and Jongsoo Park and Joseph Rocca and Joshua Johnstun and Joshua Saxe and Junteng Jia and Kalyan Vasuden Alwala and Karthik Prasad and Kartikeya Upasani and Kate Plawiak and Ke Li and Kenneth Heafield and Kevin Stone and Khalid El-Arini and Krithika Iyer and Kshitiz Malik and Kuenley Chiu and Kunal Bhalla and Kushal Lakhotia and Lauren Rantala-Yeary and Laurens van der Maaten and Lawrence Chen and Liang Tan and Liz Jenkins and Louis Martin and Lovish Madaan and Lubo Malo and Lukas Blecher and Lukas Landzaat and Luke de Oliveira and Madeline Muzzi and Mahesh Pasupuleti and Mannat Singh and Manohar Paluri and Marcin Kardas and Maria Tsimpoukelli and Mathew Oldham and Mathieu Rita and Maya Pavlova and Melanie Kambadur and Mike Lewis and Min Si and Mitesh Kumar Singh and Mona Hassan and Naman Goyal and Narjes Torabi and Nikolay Bashlykov and Nikolay Bogoychev and Niladri Chatterji and Ning Zhang and Olivier Duchenne and Onur Çelebi and Patrick Alrassy and Pengchuan Zhang and Pengwei Li and Petar Vasic and Peter Weng and Prajjwal Bhargava and Pratik Dubal and Praveen Krishnan and Punit Singh Koura and Puxin Xu and Qing He and Qingxiao Dong and Ragavan Srinivasan and Raj Ganapathy and Ramon Calderer and Ricardo Silveira Cabral and Robert Stojnic and Roberta Raileanu and Rohan Maheswari and Rohit Girdhar and Rohit Patel and Romain Sauvestre and Ronnie Polidoro and Roshan Sumbaly and Ross Taylor and Ruan Silva and Rui Hou and Rui Wang and Saghar Hosseini and Sahana Chennabasappa and Sanjay Singh and Sean Bell and Seohyun Sonia Kim and Sergey Edunov and Shaoliang Nie and Sharan Narang and Sharath Raparthy and Sheng Shen and Shengye Wan and Shruti Bhosale and Shun Zhang and Simon Vandenhende and Soumya Batra and Spencer Whitman and Sten Sootla and Stephane Collot and Suchin Gururangan and Sydney Borodinsky and Tamar Herman and Tara Fowler and Tarek Sheasha and Thomas Georgiou and Thomas Scialom and Tobias Speckbacher and Todor Mihaylov and Tong Xiao and Ujjwal Karn and Vedanuj Goswami and Vibhor Gupta and Vignesh Ramanathan and Viktor Kerkez and Vincent Gonguet and Virginie Do and Vish Vogeti and Vítor Albiero and Vladan Petrovic and Weiwei Chu and Wenhan Xiong and Wenyin Fu and Whitney Meers and Xavier Martinet and Xiaodong Wang and Xiaofang Wang and Xiaoqing Ellen Tan and Xide Xia and Xinfeng Xie and Xuchao Jia and Xuewei Wang and Yaelle Goldschlag and Yashesh Gaur and Yasmine Babaei and Yi Wen and Yiwen Song and Yuchen Zhang and Yue Li and Yuning Mao and Zacharie Delpierre Coudert and Zheng Yan and Zhengxing Chen and Zoe Papakipos and Aaditya Singh and Aayushi Srivastava and Abha Jain and Adam Kelsey and Adam Shajnfeld and Adithya Gangidi and Adolfo Victoria and Ahuva Goldstand and Ajay Menon and Ajay Sharma and Alex Boesenberg and Alexei Baevski and Allie Feinstein and Amanda Kallet and Amit Sangani and Amos Teo and Anam Yunus and Andrei Lupu and Andres Alvarado and Andrew Caples and Andrew Gu and Andrew Ho and Andrew Poulton and Andrew Ryan and Ankit Ramchandani and Annie Dong and Annie Franco and Anuj Goyal and Aparajita Saraf and Arkabandhu Chowdhury and Ashley Gabriel and Ashwin Bharambe and Assaf Eisenman and Azadeh Yazdan and Beau James and Ben Maurer and Benjamin Leonhardi and Bernie Huang and Beth Loyd and Beto De Paola and Bhargavi Paranjape and Bing Liu and Bo Wu and Boyu Ni and Braden Hancock and Bram Wasti and Brandon Spence and Brani Stojkovic and Brian Gamido and Britt Montalvo and Carl Parker and Carly Burton and Catalina Mejia and Ce Liu and Changhan Wang and Changkyu Kim and Chao Zhou and Chester Hu and Ching-Hsiang Chu and Chris Cai and Chris Tindal and Christoph Feichtenhofer and Cynthia Gao and Damon Civin and Dana Beaty and Daniel Kreymer and Daniel Li and David Adkins and David Xu and Davide Testuggine and Delia David and Devi Parikh and Diana Liskovich and Didem Foss and Dingkang Wang and Duc Le and Dustin Holland and Edward Dowling and Eissa Jamil and Elaine Montgomery and Eleonora Presani and Emily Hahn and Emily Wood and Eric-Tuan Le and Erik Brinkman and Esteban Arcaute and Evan Dunbar and Evan Smothers and Fei Sun and Felix Kreuk and Feng Tian and Filippos Kokkinos and Firat Ozgenel and Francesco Caggioni and Frank Kanayet and Frank Seide and Gabriela Medina Florez and Gabriella Schwarz and Gada Badeer and Georgia Swee and Gil Halpern and Grant Herman and Grigory Sizov and Guangyi and Zhang and Guna Lakshminarayanan and Hakan Inan and Hamid Shojanazeri and Han Zou and Hannah Wang and Hanwen Zha and Haroun Habeeb and Harrison Rudolph and Helen Suk and Henry Aspegren and Hunter Goldman and Hongyuan Zhan and Ibrahim Damlaj and Igor Molybog and Igor Tufanov and Ilias Leontiadis and Irina-Elena Veliche and Itai Gat and Jake Weissman and James Geboski and James Kohli and Janice Lam and Japhet Asher and Jean-Baptiste Gaya and Jeff Marcus and Jeff Tang and Jennifer Chan and Jenny Zhen and Jeremy Reizenstein and Jeremy Teboul and Jessica Zhong and Jian Jin and Jingyi Yang and Joe Cummings and Jon Carvill and Jon Shepard and Jonathan McPhie and Jonathan Torres and Josh Ginsburg and Junjie Wang and Kai Wu and Kam Hou U and Karan Saxena and Kartikay Khandelwal and Katayoun Zand and Kathy Matosich and Kaushik Veeraraghavan and Kelly Michelena and Keqian Li and Kiran Jagadeesh and Kun Huang and Kunal Chawla and Kyle Huang and Lailin Chen and Lakshya Garg and Lavender A and Leandro Silva and Lee Bell and Lei Zhang and Liangpeng Guo and Licheng Yu and Liron Moshkovich and Luca Wehrstedt and Madian Khabsa and Manav Avalani and Manish Bhatt and Martynas Mankus and Matan Hasson and Matthew Lennie and Matthias Reso and Maxim Groshev and Maxim Naumov and Maya Lathi and Meghan Keneally and Miao Liu and Michael L. Seltzer and Michal Valko and Michelle Restrepo and Mihir Patel and Mik Vyatskov and Mikayel Samvelyan and Mike Clark and Mike Macey and Mike Wang and Miquel Jubert Hermoso and Mo Metanat and Mohammad Rastegari and Munish Bansal and Nandhini Santhanam and Natascha Parks and Natasha White and Navyata Bawa and Nayan Singhal and Nick Egebo and Nicolas Usunier and Nikhil Mehta and Nikolay Pavlovich Laptev and Ning Dong and Norman Cheng and Oleg Chernoguz and Olivia Hart and Omkar Salpekar and Ozlem Kalinli and Parkin Kent and Parth Parekh and Paul Saab and Pavan Balaji and Pedro Rittner and Philip Bontrager and Pierre Roux and Piotr Dollar and Polina Zvyagina and Prashant Ratanchandani and Pritish Yuvraj and Qian Liang and Rachad Alao and Rachel Rodriguez and Rafi Ayub and Raghotham Murthy and Raghu Nayani and Rahul Mitra and Rangaprabhu Parthasarathy and Raymond Li and Rebekkah Hogan and Robin Battey and Rocky Wang and Russ Howes and Ruty Rinott and Sachin Mehta and Sachin Siby and Sai Jayesh Bondu and Samyak Datta and Sara Chugh and Sara Hunt and Sargun Dhillon and Sasha Sidorov and Satadru Pan and Saurabh Mahajan and Saurabh Verma and Seiji Yamamoto and Sharadh Ramaswamy and Shaun Lindsay and Shaun Lindsay and Sheng Feng and Shenghao Lin and Shengxin Cindy Zha and Shishir Patil and Shiva Shankar and Shuqiang Zhang and Shuqiang Zhang and Sinong Wang and Sneha Agarwal and Soji Sajuyigbe and Soumith Chintala and Stephanie Max and Stephen Chen and Steve Kehoe and Steve Satterfield and Sudarshan Govindaprasad and Sumit Gupta and Summer Deng and Sungmin Cho and Sunny Virk and Suraj Subramanian and Sy Choudhury and Sydney Goldman and Tal Remez and Tamar Glaser and Tamara Best and Thilo Koehler and Thomas Robinson and Tianhe Li and Tianjun Zhang and Tim Matthews and Timothy Chou and Tzook Shaked and Varun Vontimitta and Victoria Ajayi and Victoria Montanez and Vijai Mohan and Vinay Satish Kumar and Vishal Mangla and Vlad Ionescu and Vlad Poenaru and Vlad Tiberiu Mihailescu and Vladimir Ivanov and Wei Li and Wenchen Wang and Wenwen Jiang and Wes Bouaziz and Will Constable and Xiaocheng Tang and Xiaojian Wu and Xiaolan Wang and Xilun Wu and Xinbo Gao and Yaniv Kleinman and Yanjun Chen and Ye Hu and Ye Jia and Ye Qi and Yenda Li and Yilin Zhang and Ying Zhang and Yossi Adi and Youngjin Nam and Yu and Wang and Yu Zhao and Yuchen Hao and Yundi Qian and Yunlu Li and Yuzi He and Zach Rait and Zachary DeVito and Zef Rosnbrick and Zhaoduo Wen and Zhenyu Yang and Zhiwei Zhao and Zhiyu Ma},
      year={2024},
      eprint={2407.21783},
      archivePrefix={arXiv},
      primaryClass={cs.AI},
      url={https://doi.org/10.48550/arXiv.2407.21783}, 
}

@misc{biderman2023pythia,
      title={Pythia: A Suite for Analyzing Large Language Models Across Training and Scaling}, 
      author={Stella Biderman and Hailey Schoelkopf and Quentin Anthony and Herbie Bradley and Kyle O'Brien and Eric Hallahan and Mohammad Aflah Khan and Shivanshu Purohit and USVSN Sai Prashanth and Edward Raff and Aviya Skowron and Lintang Sutawika and Oskar van der Wal},
      year={2023},
      eprint={2304.01373},
      archivePrefix={arXiv},
      primaryClass={cs.CL},
      url={https://doi.org/10.48550/arXiv.2304.01373}, 
}

@book{foucart2013compressivesensing,
  author    = {Simon Foucart and Holger Rauhut},
  title     = {A Mathematical Introduction to Compressive Sensing},
  series    = {Applied and Numerical Harmonic Analysis},
  publisher = {Birkh{\"a}user},
  address   = {New York},
  year      = {2013},
  url       = {https://doi.org/10.1007/978-0-8176-4948-7},
}

@book{wainwright2019highdimensional,
  title={High-Dimensional Statistics: A Non-Asymptotic Viewpoint},
  author={Wainwright, Martin J.},
  series={Cambridge Series in Statistical and Probabilistic Mathematics},
  year={2019},
  publisher={Cambridge University Press},
  url={https://doi.org/10.1017/9781108627771}
}

@misc{bussmann2024batchtopk,
      title={{BatchTopK} Sparse Autoencoders}, 
      author={Bart Bussmann and Patrick Leask and Neel Nanda},
      year={2024},
      eprint={2412.06410},
      archivePrefix={arXiv},
      primaryClass={cs.LG},
      url={https://doi.org/10.48550/arXiv.2412.06410}, 
}

@misc{ayonrinde2024feature,
      title={Adaptive Sparse Allocation with Mutual Choice \& Feature Choice Sparse Autoencoders}, 
      author={Kola Ayonrinde},
      year={2024},
      eprint={2411.02124},
      archivePrefix={arXiv},
      primaryClass={cs.LG},
      url={https://doi.org/10.48550/arXiv.2411.02124}, 
}

@misc{lee2025evaluating,
      title={Evaluating and Designing Sparse Autoencoders by Approximating Quasi-Orthogonality}, 
      author={Sewoong Lee and Adam Davies and Marc E. Canby and Julia Hockenmaier},
      year={2025},
      eprint={2503.24277},
      archivePrefix={arXiv},
      primaryClass={cs.LG},
      url={https://doi.org/10.48550/arXiv.2503.24277}, 
}

@misc{chanin2025absorption,
      title={A is for Absorption: Studying Feature Splitting and Absorption in Sparse Autoencoders}, 
      author={David Chanin and James Wilken-Smith and Tomáš Dulka and Hardik Bhatnagar and Satvik Golechha and Joseph Bloom},
      year={2025},
      eprint={2409.14507},
      archivePrefix={arXiv},
      primaryClass={cs.CL},
      url={https://doi.org/10.48550/arXiv.2409.14507}, 
}

@article{tibshirani1996lasso,
  title   = {Regression Shrinkage and Selection via the {L}asso},
  author  = {Tibshirani, Robert},
  journal = {Journal of the Royal Statistical Society: Series B (Methodological)},
  volume  = {58},
  number  = {1},
  pages   = {267--288},
  year    = {1996},
  url     = {https://doi.org/10.1111/j.2517-6161.1996.tb02080.x},
}

@article{zhao2006consistency,
  author  = {Peng Zhao and Bin Yu},
  title   = {On Model Selection Consistency of {L}asso},
  journal = {Journal of Machine Learning Research},
  year    = {2006},
  volume  = {7},
  number  = {90},
  pages   = {2541--2563},
  url     = {http://jmlr.org/papers/v7/zhao06a.html}
}

@article{zou2005elasticnet,
    author = {Zou, Hui and Hastie, Trevor},
    title = {Regularization and variable selection via the elastic net},
    journal = {Journal of the Royal Statistical Society: Series B (Statistical Methodology)},
    volume = {67},
    number = {2},
    pages = {301-320},
    url = {https://doi.org/10.1111/j.1467-9868.2005.00503.x},
    year = {2005}
}

@article{zou2006adaptivelasso,
    author = {Hui Zou},
    title = {The Adaptive {L}asso and Its Oracle Properties},
    journal = {Journal of the American Statistical Association},
    volume = {101},
    number = {476},
    pages = {1418--1429},
    year = {2006},
    publisher = {Taylor \& Francis},
    url = {https://doi.org/10.1198/016214506000000735},
}

@article{mairal2010online,
  title={Online learning for matrix factorization and sparse coding},
  author={Mairal, Julien and Bach, Francis and Ponce, Jean and Sapiro, Guillermo},
  journal={Journal of Machine Learning Research},
  volume={11},
  number={1},
  year={2010},
  url     = {http://jmlr.org/papers/v11/mairal10a.html}
}

@article{zou2009adaptiveelasticnet,
  author  = {Zou, Hui and Zhang, Hao Helen},
  title   = {On the adaptive elastic-net with a diverging number of parameters},
  journal = {The Annals of Statistics},
  year    = {2009},
  volume  = {37},
  number  = {4},
  pages   = {1733--1751},
  url     = {https://doi.org/10.1214/08-AOS625}
}

@article{hoffman1952approximate,
  title={On Approximate Solutions of Systems of Linear Inequalities},
  author={Hoffman, Alan J.},
  journal={Journal of Research of the National Bureau of Standards},
  volume={49},
  number={4},
  pages={263--265},
  year={1952},
  url={https://nvlpubs.nist.gov/nistpubs/jres/049/4/v49.n04.a05.pdf}
}

@article{robinson1981polyhedral,
  title={Some continuity properties of polyhedral multifunctions},
  author={Robinson, Stephen M.},
  journal={Mathematical Programming at Oberwolfach},
  volume={14},
  pages={206--214},
  year={1981},
  publisher={Springer Berlin Heidelberg},
  url={https://doi.org/10.1007/BFb0120929}
}

@inproceedings{gregor2010lista,
  title     = {Learning Fast Approximations of Sparse Coding},
  author    = {Gregor, Karol and LeCun, Yann},
  booktitle = {Proceedings of the 27th International Conference on Machine Learning},
  series    = {ICML},
  year      = {2010},
  url       = {https://icml.cc/Conferences/2010/papers/449.pdf}
}

@misc{chen2021lista,
      title={Hyperparameter Tuning is All You Need for {LISTA}}, 
      author={Xiaohan Chen and Jialin Liu and Zhangyang Wang and Wotao Yin},
      year={2021},
      eprint={2110.15900},
      archivePrefix={arXiv},
      primaryClass={cs.LG},
      url={https://doi.org/10.48550/arXiv.2110.15900}, 
}

@article{piantadosi2014zipf,
  title   = {Zipf's Word Frequency Law in Natural Language: A Critical Review and Future Directions},
  author  = {Piantadosi, Steven T.},
  journal = {Psychonomic Bulletin \& Review},
  volume  = {21},
  number  = {5},
  pages   = {1112--1130},
  year    = {2014},
  url     = {https://doi.org/10.3758/s13423-014-0585-6}
}

@misc{bartlett2017spectral,
      title={Spectrally-normalized margin bounds for neural networks}, 
      author={Peter Bartlett and Dylan J. Foster and Matus Telgarsky},
      year={2017},
      eprint={1706.08498},
      archivePrefix={arXiv},
      primaryClass={cs.LG},
      url={https://doi.org/10.48550/arXiv.1706.08498}, 
}

@misc{neyshabur2017exploring,
      title={Exploring Generalization in Deep Learning}, 
      author={Behnam Neyshabur and Srinadh Bhojanapalli and David McAllester and Nathan Srebro},
      year={2017},
      eprint={1706.08947},
      archivePrefix={arXiv},
      primaryClass={cs.LG},
      url={https://doi.org/10.48550/arXiv.1706.08947}, 
}

@misc{keskar2017large,
      title={On Large-Batch Training for Deep Learning: Generalization Gap and Sharp Minima}, 
      author={Nitish Shirish Keskar and Dheevatsa Mudigere and Jorge Nocedal and Mikhail Smelyanskiy and Ping Tak Peter Tang},
      year={2017},
      eprint={1609.04836},
      archivePrefix={arXiv},
      primaryClass={cs.LG},
      url={https://doi.org/10.48550/arXiv.1609.04836}, 
}

@article{hochreiter1997flat,
  title   = {Flat Minima},
  author  = {Hochreiter, Sepp and Schmidhuber, J{\"u}rgen},
  journal = {Neural Computation},
  volume  = {9},
  number  = {1},
  pages   = {1--42},
  year    = {1997},
  url     = {https://doi.org/10.1162/neco.1997.9.1.1}
}

@misc{nanda2022glossary,
  title={A Comprehensive Mechanistic Interpretability Explainer \& Glossary},
  author={Nanda, Neel},
  year={2022},
  url={https://www.neelnanda.io/mechanistic-interpretability/glossary}
}

@misc{li2024optimal,
      title={Optimal ablation for interpretability}, 
      author={Maximilian Li and Lucas Janson},
      year={2024},
      eprint={2409.09951},
      archivePrefix={arXiv},
      primaryClass={cs.LG},
      url={https://doi.org/10.48550/arXiv.2409.09951}, 
}

@misc{pochinkov2024investigating,
      title={Investigating Neuron Ablation in Attention Heads: The Case for Peak Activation Centering}, 
      author={Nicholas Pochinkov and Ben Pasero and Skylar Shibayama},
      year={2024},
      eprint={2408.17322},
      archivePrefix={arXiv},
      primaryClass={cs.LG},
      url={https://doi.org/10.48550/arXiv.2408.17322}, 
}

@misc{su2023rope,
      title={{RoFormer}: Enhanced Transformer with Rotary Position Embedding}, 
      author={Jianlin Su and Yu Lu and Shengfeng Pan and Ahmed Murtadha and Bo Wen and Yunfeng Liu},
      year={2023},
      eprint={2104.09864},
      archivePrefix={arXiv},
      primaryClass={cs.CL},
      url={https://doi.org/10.48550/arXiv.2104.09864}, 
}

@misc{gao2020pile800gbdatasetdiverse,
      title={The {P}ile: An 800GB Dataset of Diverse Text for Language Modeling}, 
      author={Leo Gao and Stella Biderman and Sid Black and Laurence Golding and Travis Hoppe and Charles Foster and Jason Phang and Horace He and Anish Thite and Noa Nabeshima and Shawn Presser and Connor Leahy},
      year={2020},
      eprint={2101.00027},
      archivePrefix={arXiv},
      primaryClass={cs.CL},
      url={https://doi.org/10.48550/arXiv.2101.00027}, 
}

@book{rockafellar1998variational,
  title={Variational analysis},
  author={Rockafellar, Tyrrell and Wets, Roger},
  year={1998},
  publisher={Springer},
  url={https://doi.org/10.1007/978-3-642-02431-3}
}

@misc{liu2025olmotracetracinglanguagemodel,
      title={{OLMoTrace}: Tracing Language Model Outputs Back to Trillions of Training Tokens}, 
      author={Jiacheng Liu and Taylor Blanton and Yanai Elazar and Sewon Min and YenSung Chen and Arnavi Chheda-Kothary and Huy Tran and Byron Bischoff and Eric Marsh and Michael Schmitz and Cassidy Trier and Aaron Sarnat and Jenna James and Jon Borchardt and Bailey Kuehl and Evie Cheng and Karen Farley and Sruthi Sreeram and Taira Anderson and David Albright and Carissa Schoenick and Luca Soldaini and Dirk Groeneveld and Rock Yuren Pang and Pang Wei Koh and Noah A. Smith and Sophie Lebrecht and Yejin Choi and Hannaneh Hajishirzi and Ali Farhadi and Jesse Dodge},
      year={2025},
      eprint={2504.07096},
      archivePrefix={arXiv},
      primaryClass={cs.CL},
      url={https://doi.org/10.48550/arXiv.2504.07096}, 
}

@misc{yang2022tensorprogramsvtuning,
      title={Tensor Programs V: Tuning Large Neural Networks via Zero-Shot Hyperparameter Transfer}, 
      author={Greg Yang and Edward J. Hu and Igor Babuschkin and Szymon Sidor and Xiaodong Liu and David Farhi and Nick Ryder and Jakub Pachocki and Weizhu Chen and Jianfeng Gao},
      year={2022},
      eprint={2203.03466},
      archivePrefix={arXiv},
      primaryClass={cs.LG},
      url={https://doi.org/10.48550/arXiv.2203.03466}, 
}

@misc{fan2008sureindependencescreeningultrahigh,
      title={Sure Independence Screening for Ultra-High Dimensional Feature Space}, 
      author={Jianqing Fan and Jinchi Lv},
      year={2008},
      eprint={math/0612857},
      archivePrefix={arXiv},
      primaryClass={math.ST},
      url={https://doi.org/10.48550/arXiv.math/0612857}, 
}

@misc{lieberum2024gemma,
      title={Gemma Scope: Open Sparse Autoencoders Everywhere All At Once on Gemma 2}, 
      author={Tom Lieberum and Senthooran Rajamanoharan and Arthur Conmy and Lewis Smith and Nicolas Sonnerat and Vikrant Varma and János Kramár and Anca Dragan and Rohin Shah and Neel Nanda},
      year={2024},
      eprint={2408.05147},
      archivePrefix={arXiv},
      primaryClass={cs.LG},
      url={https://doi.org/10.48550/arXiv.2408.05147}, 
}

@misc{rajamanoharan2024jumping,
      title={Jumping Ahead: Improving Reconstruction Fidelity with JumpReLU Sparse Autoencoders}, 
      author={Senthooran Rajamanoharan and Tom Lieberum and Nicolas Sonnerat and Arthur Conmy and Vikrant Varma and János Kramár and Neel Nanda},
      year={2024},
      eprint={2407.14435},
      archivePrefix={arXiv},
      primaryClass={cs.LG},
      url={https://doi.org/10.48550/arXiv.2407.14435}, 
}

@misc{bussmann2025matryoshka,
      title={Learning Multi-Level Features with Matryoshka Sparse Autoencoders}, 
      author={Bart Bussmann and Noa Nabeshima and Adam Karvonen and Neel Nanda},
      year={2025},
      eprint={2503.17547},
      archivePrefix={arXiv},
      primaryClass={cs.LG},
      url={https://doi.org/10.48550/arXiv.2503.17547}, 
}

@misc{martinlinares2025attribution,
      title={Attribution-Guided Distillation of Matryoshka Sparse Autoencoders}, 
      author={Cristina P. Martin-Linares and Jonathan P. Ling},
      year={2025},
      eprint={2512.24975},
      archivePrefix={arXiv},
      primaryClass={cs.LG},
      url={https://doi.org/10.48550/arXiv.2512.24975}, 
}


\clearpage
\appendix

\section{Theoretical Guarantees: Polyhedral Stability and Oracle Consistency}
\label{sec:theory}

In this section, we formalize the statistical and stability guarantees of the AEN-SAE. For notational simplicity during the subsequent matrix calculus derivations, we assume the input residual stream vector $x$ has been precentered by the explicit decoder bias ($x \leftarrow x - b_{\text{dec}}$). This allows the reconstruction error to be written more cleanly as $\|x - Dh\|_2^2 / 2$ without loss of generality.

\subsection{Polyhedral and Lipschitz Stability}

\begin{theorem}[Active Curvature Lower Bound]
\label{thm:stability}
Fix a latent representation $h\in\mathbb{R}_{+}^{d_{\text{dict}}}$ and define its active features by $\mathcal A:=\{i:h_i \ne 0\}$. Treating the adaptive weights as fixed,
the Hessian with respect to the active latent representations $\nabla^2_{h_{\mathcal{A}}} \mathcal{L}_{AEN}$ satisfies:
\begin{equation*}
    \lambda_{\min}(\nabla^2_{h_{\mathcal{A}}} \mathcal{L}_{AEN}) = \lambda_{\min}(D_{\mathcal{A}}^\top D_{\mathcal{A}}) + 2\lambda_2 \ge 2\lambda_2 > 0.
\end{equation*}
\end{theorem}

\begin{proof}
For a fixed active set $\mathcal{A}$, the AEN-SAE objective evaluates to:
\begin{equation*}
    \mathcal{L}_{\mathcal{A}}(h_{\mathcal{A}}) = \frac{1}{2} \|x - D_{\mathcal{A}} h_{\mathcal{A}}\|_2^2 + \lambda_1 \sum_{i \in \mathcal{A}} w_i |h_i| + \lambda_2 \|h_{\mathcal{A}}\|_2^2.
\end{equation*}

To compute the Hessian, we first derive the gradient with respect to the active latent vector $h_{\mathcal{A}}$. Since the absolute value function is locally linear, and its derivative is strictly defined by the sign function, the standard matrix calculus gives the gradient given by:
\begin{equation*}
    \nabla_{h_{\mathcal{A}}} \mathcal{L}_{\mathcal{A}} = \underbrace{-D_{\mathcal{A}}^\top (x - D_{\mathcal{A}} h_{\mathcal{A}})}_{\text{Reconstruction gradient}} + \underbrace{\lambda_1 W_{\mathcal{A}} \text{sign}(h_{\mathcal{A}})}_{\ell_1 \text{ gradient}} + \underbrace{2\lambda_2 h_{\mathcal{A}}}_{\ell_2 \text{ gradient}},
\end{equation*}
where $W_{\mathcal{A}}$ is a diagonal matrix of the active adaptive weights.

Taking the derivative of this gradient with respect to $h_{\mathcal{A}}$ yields the Hessian. Distributing the transpose in the reconstruction term yields a constant $D_{\mathcal{A}}^\top D_{\mathcal{A}}$ component. Because the $\ell_1$ gradient term ($\lambda_1 W_{\mathcal{A}} \cdot \text{sign}(h_{\mathcal{A}})$) is constant within this active orthant, its second derivative vanishes. The Hessian is therefore derived from the remaining quadratic and $\ell_2$ terms:
\begin{equation*}
    \nabla^2_{h_{\mathcal{A}}} \mathcal{L}_{\mathcal{A}} = D_{\mathcal{A}}^\top D_{\mathcal{A}} + 2\lambda_2 I_{|\mathcal{A}|}.
\end{equation*}

Applying the Rayleigh quotient formula, the minimum eigenvalue is:
\begin{equation*}
    \lambda_{\min}(D_{\mathcal{A}}^\top D_{\mathcal{A}} + 2\lambda_2 I_{|\mathcal{A}|}) = \lambda_{\min}(D_{\mathcal{A}}^\top D_{\mathcal{A}}) + 2\lambda_2.
\end{equation*}

Since the Gram matrix $D_{\mathcal{A}}^\top D_{\mathcal{A}}$ is positive semi-definite by definition ($\lambda_{\min}(D_{\mathcal{A}}^\top D_{\mathcal{A}}) \ge 0$), the Hessian's eigenvalues are strictly bounded below by $2\lambda_2$, which completes the proof.
\end{proof}

As established in the main text, the Hoffman constant for the active polyhedra scales inversely with the minimum eigenvalue of the active Gram matrix. Theorem~\ref{thm:stability} explicitly guarantees that the Hoffman bound is capped by $\mathcal{O}(1/2\lambda_2)$. This prevents the activation polyhedra from collapsing into degenerate valleys, avoiding a key geometric source of instability associated with feature starvation.

Furthermore, this $\ell_2$ curvature has a secondary desirable effect in deep learning generalization. By avoiding discontinuous ravines, the $\ell_2$ penalty forces the optimization to converge toward broader, flatter local minima in the representation space~\citep{hochreiter1997flat, keskar2017large}. As we will formalize in Theorem~\ref{thm:lipschitz}, this flat geometry explicitly bounds the Lipschitz constant of the target mapping. Classical statistical learning theory dictates that such Lipschitz regularization is strictly necessary to close the generalization gap, ensuring the learned dictionary maintains high representational fidelity on unseen text~\citep{bartlett2017spectral, neyshabur2017exploring}.

\begin{theorem}[Lipschitz Stability of the Latent Encoder Mapping]
\label{thm:lipschitz}
Let $h^{\dagger}(x)$ be the optimal sparse code for an input $x$ under the AEN-SAE objective:
\begin{equation*}
    h^{\dagger}(x)=\argmin_{h(x)\in\mathbb{R}^{d_{\text{dict}}}_{+}}
    \left[ \frac{1}{2} \|x - D h(x) \|_2^2 + \lambda_1 \sum_{i=1}^{d_{\text{dict}}} w_i |h_i(x)| + \lambda_2 \|h(x)\|_2^2 \right],
\end{equation*}
where the adaptive weights are treated as fixed.
Then the target mapping from the residual stream $x$ to the latent space $h$ is Lipschitz continuous with constant $L \le 1/(2\sqrt{2\lambda_2})$:
\begin{equation*}
    \|h^{\dagger}(x_1) - h^{\dagger}(x_2)\|_2 \le \frac{1}{2\sqrt{2\lambda_2}} \|x_1 - x_2\|_2.
\end{equation*}
\end{theorem}

\begin{proof}
    Fix any $x_{1}\ne x_{2}$.
    Let $\delta_{h\ge 0}$ denote the indicator function such that $\delta_{h\ge 0}=0$ for $h\ge 0$ and $\delta_{h\ge 0}=\infty$ otherwise.
    The first order optimality condition (e.g.,~\citep[Theorem 8.15]{rockafellar1998variational}) implies that there exist 
    subgradients $s_{1}$ and $s_{2}$ of $\sum_{i=1}^{d_{dict}}w_{i}|h_{i}|+\delta_{h\ge 0}$ (with respect to $h=(h_{1},\ldots,h_{d_{dict}})^{\top}$) at $h^{\dagger}(x_{1})$ and $h^{\dagger}(x_{2})$ for which
    \begin{align*}
        D^{\top}\{Dh^{\dagger}(x_{1})-x_{1}\}+2\lambda_{2}h^{\dagger}(x_{1}) + \lambda_{1}s_{1} &= 0, \\
        D^{\top}\{Dh^{\dagger}(x_{2})-x_{2}\}+2\lambda_{2}h^{\dagger}(x_{2}) + \lambda_{1}s_{2} &= 0.
    \end{align*}
    These give
    \[
    (D^{\top}D+2\lambda_{2}I)\{h^{\dagger}(x_{1})-h^{\dagger}(x_{2})\}-D^{\top}(x_{1}-x_{2}) +\lambda_{1}(s_{1}-s_{2})=0.
    \]
    Taking the inner product with respect to $\Delta h:=h^{\dagger}(x_{1})-h^{\dagger}(x_{2})$ gives
    \[
    (\Delta h)^{\top}(D^{\top}D+2\lambda_{2}I)\Delta h-
    \langle D^{\top}(x_{1}-x_{2}), \Delta h)\rangle  
    +\lambda_{1}\langle s_{1}-s_{2}, \Delta h)\rangle=0.
    \]
    Since the convexity of $\sum_{i=1}^{d_{dict}}w_{i}|h_{i}|+\delta_{h\ge 0}$ with respect to $h$ implies
    the monotonicity of the subdifferential $\langle s_{1}-s_{2}, \Delta h)\rangle =\langle s_{1}-s_{2}, h^{\dagger}(x_{1})-h^{\dagger}(x_{2})\rangle \ge 0$ 
    (e.g.,~\citep[Theorem 12.17]{rockafellar1998variational}),
    this is further reduced to
    \[
    (\Delta h)^{\top}(D^{\top}D+2\lambda_{2}I)\Delta h
    - \langle D^{\top}(x_{1}-x_{2}), \Delta h)\rangle \le 0,
    \]
    and, together with the Cauchy--Schwarz inequality, we have
    \[
    \|D\Delta h\|^{2}_{2}+
    2\lambda_{2}\|\Delta h\|_{2}^{2}
    - \|x_{1}-x_{2}\|_{2} \|D \Delta h\|_{2} \le 0.
    \]
    Completing the square gives
    \[
    \{\|D \Delta h \|_{2}-\|x_{1}-x_{2}\|_{2}/2\}^{2} - \|x_{1}-x_{2}\|_{2}^{2}/4 + 2\lambda_{2}\|\Delta h\|_{2}^{2} \le 0,
    \]
    implying that $\|h^{\dagger}(x_{1})-h^{\dagger}(x_{2})\|_{2}^{2} \le \|x_{1}-x_{2}\|_{2}^{2}/(8\lambda_{2})$ and thus completing the proof.
\end{proof}

The bound established in Theorem~\ref{thm:lipschitz} exposes the mechanical failure of standard SAEs. In a standard $\ell_1$ framework ($\lambda_2 = 0$), 
the Lipshitz stability is governed by the inverse of the minimum singular value of the active dictionary $\sigma_{\min}(D_{\mathcal A})$. As mutual coherence increases (as in the spiked model~\citep{wainwright2019highdimensional}), $\sigma_{\min} \to 0$, causing the Hoffman constant of the solution map to diverge to infinity. 

Simultaneously, the expressive capacity of a shallow linear encoder $h(x) = \text{ReLU}(W_{\text{enc}}x + b_{\text{enc}})$ is strictly limited by the spectral norm of its weight matrix, $\|W_{\text{enc}}\|_2$. To successfully approximate the discontinuous $\ell_1$ target manifold, the optimization would require $\|W_{\text{enc}}\|_2 \to \infty$. Because gradient descent naturally constrains weight norms, the encoder structurally lacks the spectral capacity to track these discontinuities. Unable to map the target, the network effectively abandons the unstable features by pushing their pre-activations below the ReLU threshold, causing irreversible feature starvation. By bounding the target Lipschitz constant to $1/(2\sqrt{2\lambda_2})$, the AEN-SAE ensures the target manifold remains comfortably within the spectral capacity of $W_{\text{enc}}$.

Finally, the result of Theorem~\ref{thm:lipschitz} would suggest that the map becomes more stable and easier to learn as $\lambda_2$ is increased. While this is true, it should be noted that, when $\lambda_2$ is too high relative to $\lambda_1$, the loss will function closer to Ridge regression with a grouping effect rather than LASSO with feature selection. That is, representations will remain dense and not allow for any monosemantic interpretation. 

\subsection{Assumptions for Oracle Selection Consistency}

In classical high-dimensional statistics, standard $\ell_1$ recovery relies on the strict IRC~\citep{wainwright2019highdimensional, zhao2006consistency}, which assumes that noise features are largely orthogonal to true signal features. In overcomplete LLM dictionaries ($d_{\text{dict}} \gg d_{\text{model}}$), features are naturally highly collinear to capture polysemantic nuance. This inevitably violates the IRC, causing standard $\ell_1$ SAEs to collapse and misidentify features.

Further, the true active support is assumed to be static across the dataset in classical sparse regression. 
In contrast, LLM residual streams exhibit a dynamic \textit{local} sparsity, where 
for an encoder $h$,
the active features $\mathcal{A}(x):=\{i:|h_{i}(x)|>0\}$ change for every input token $x$. To evaluate the structural health of the dictionary, we define the \textit{global active support} $\mathcal{A}_{\text{global}} = \{i : \mathbb{E}_{x}[|h_i (x)|] > 0\}$ as the set of all statistically valid semantic concepts across the data distribution. Noise artifacts and dead features inherently belong to the inactive set (where $\mathbb{E}_{x}[|h_i (x)|] = 0$).

A primary theoretical advantage of the AEN-SAE is that it achieves oracle selection consistency for the global active support
without requiring the IRC~\citep{zou2006adaptivelasso, zou2009adaptiveelasticnet}. By conditioning the problem via the $\ell_2$ anchor and dynamically penalizing noise via adaptive weights, we reduce the theoretical requirements to two assumptions.
Let $\mathcal{A}^{*}_{\text{global}}$ be the true global active support.
Let $h^{*}:\mathbb{R}^{d_{\text{model}}}\to\mathbb{R}^{d_{\text{dict}}}$ be the oracle sparse encoder defined by
\[
h^{*}(x) = \argmin_{h(x)\in\mathbb{R}_{+}^{d_{\text{dict}}}:
h_{i}(x)=0, i\notin \mathcal{A}^{*}_{\text{global}}
}
\left[\frac{1}{2}\|x- Dh(x)\|_{2}^{2}
+\lambda_{2}\|h(x)\|_{2}^{2}
\right]
\]

\begin{assumption}[Signal Lower Bound]
\label{assum:signal}
The true active features have expected magnitudes strictly bounded away from zero: $\min_{i \in \mathcal{A}^*_{\text{global}}} \mathbb{E}_{x}[|h^*_i (x)|] \ge c > 0$.
\end{assumption}

Assumption~\ref{assum:signal} is quite mild, and in fact required for model identifiability to separate the active and inactive features in expectation. In mechanistic interpretability, a `true' feature represents a semantic concept that exerts a causal impact on downstream logit generation. Because transformer residual streams are high-dimensional and subject to LayerNorm, any feature with an infinitesimally small activation ($c \approx 0$) would be washed out by the baseline variance of the stream. Therefore, impactful concepts possess a macroscopic activation threshold bounded strictly away from zero.

\begin{assumption}[Timescale Separation]
\label{assum:timescale}
The dictionary updates occur on a sufficiently slow timescale relative to the EMA tracking such that, as $t \to \infty$, the EMA buffer converges to the true marginal expected activations: $\bar{h}_i^{(t)} \to \mathbb{E}_{x}[|h_i^* (x)|]$.
\end{assumption}

Assumption~\ref{assum:timescale} addresses the fundamental challenge of applying online, adaptive estimators to natural language, which follows a heavily skewed Zipfian (power-law) distribution~\citep{piantadosi2014zipf}. While raw syntactic features fire constantly, highly specific semantic features (e.g., proper nouns or niche factual knowledge) may only appear once every several million tokens. We satisfy the timescale separation requirement through two explicit mechanisms designed for this long-tail distribution. First, we hold the warmup schedule at $\rho(t) = 0$ for millions of tokens to allow the buffer to safely initialize. Second, we calibrate the EMA momentum $\beta$ such that the buffer update rate, defined as $\eta_{\text{ema}} := (1-\beta)$, is several orders of magnitude larger than the dictionary learning rate $\eta_{D}$. In our implementation, we observe empirical dictionary updates of $\mathcal{O}(10^{-7})$ per step, while our choice of $\beta=0.999$ yields a buffer rate of $10^{-3}$. Due to this, the dictionary remains quasistatic relative to the effective memory horizon, $\mathcal{H} := \frac{1}{1-\beta}$. Because the horizon $\mathcal{H} \times B$ (where $B$ is the batch size) spans tens of millions of tokens, the EMA converges to true marginals without artificially decaying healthy, rare semantic concepts to zero.

\begin{assumption}[Mild Decay Condition of the Penalties]
    \label{assum:neglibiblebiasterm}
    For the $\ell_{1}$ penalty $\lambda_{1}$
    and
    the dynamic adaptive penalty weights $w_{i}=w_{i}^{(t)}$,
    the factor $\max_{i\in\mathcal{A}^{*}_{\text{global}}}\lambda_{1}w_{i}^{(t)}$ becomes negligible as $t\to\infty$, while the factor $\min_{i\notin\mathcal{A}^{*}_{\text{global}}}\lambda_{1}w_{i}^{(t)}$ diverges as $t\to\infty$
\end{assumption}

Assumption \ref{assum:neglibiblebiasterm} is also quite mild as we discuss in Sections \ref{sec:aen} and \ref{subsec:stability-mechanisms}. In our practical online formulation, this corresponds to an asymptotic decay rate condition on $\lambda_{1}$ in the adaptive LASSO literature \citep{zou2006adaptivelasso}. As shown in the proof of
Theorem \ref{thm:oracle}, under Assumptions \ref{assum:signal}--\ref{assum:timescale},
the adaptive penalty weights satisfy the following behavior:
for $i\in\mathcal{A}^{*}_{\text{global}}$,
$w^{(t)}_{i}\to \mathbb{E}_{x}[|h^*_i(x)|]^{-\gamma}<c^{-\gamma}$,
whereas for $i\notin\mathcal{A}^{*}_{\text{global}}$, $w^{(t)}_{i}$ diverges. If the
weights are clipped to the interval $(w_{\min},w_{\max})$, then
$w^{(t)}_{i}$ reaches $w_{\max}$ for inactive coordinates. Thus, in practice,
the assumption is satisfied when $\lambda_{1}c^{-\gamma}\ll 1$ and $\lambda_{1}w_{\max}\gg 1$.

\subsection{Oracle Selection Consistency}

We now show that the idealized AEN-SAE targeted objective asymptotically identifies this global support and eliminates magnitude shrinkage for valid features, regardless of dictionary coherence.

\begin{theorem}[Global Oracle Selection and Shrinkage Elimination]
\label{thm:oracle}
Assume~\ref{assum:signal}-~\ref{assum:neglibiblebiasterm}. Let the idealized adaptive penalty weights be constructed as $\tilde{w}^{(t)}_i = (\bar{h}^{(t)}_i)^{-\gamma}$. Then, the AEN-SAE update
\begin{equation*}
    \hat{h}^{(t)} = \argmin_{h(x)\in \mathbb{R}_{+}^{d_{\text{dict}}}}\left[ \frac{1}{2} \|x - D h(x)\|_2^2 + \lambda_1 \sum_{i=1}^{d_{\text{dict}}} \tilde{w}^{(t)}_i h_i(x) + \lambda_2 \|h(x)\|_2^2 \right]
\end{equation*}
identifies the true global support with probability approaching 1. That is: under the assumption $\mathbb{E}_{x}[\|x\|_{2}^{2}]<\infty$, we have, 
for any threshold $\tau\in (0,c)$,
\begin{equation}
\label{eq: timescale identification}
    \lim_{t \to \infty} P\left(\{i : \mathbb{E}_{x}[\hat{h}^{(t)}_i(x)\mid \tilde{w}^{(t)}] > \tau \} = \mathcal{A}^*_{\text{global}}\right) = 1.
\end{equation}
Also, $\hat{h}^{(t)}$ eliminates the $\ell_{1}$-induced shrinkage: $\hat{h}^{(t)}(x)\to h^{*}(x)$.
\end{theorem}

\begin{proof}
    To establish selection consistency without the IRC, we follow several steps.
    Let $\mu^{*}_{i}:=\mathbb{E}_{x}[|h^{*}_{i}(x)|]$.
    
    \textbf{Step 1: Weight separation}.
    Fix $i\in\mathcal{A}^{*}_{\text{global}}$.
    From the signal lower bound assumption (Assumption~\ref{assum:signal}), we have $\mu^{*}_{i}>0$. The continuous mapping theorem with the timescale separation assumption (Assumption~\ref{assum:timescale}) yields
    $\tilde{w}^{(t)}_{i}=(\bar{h}^{(t)}_{i})^{-\gamma} \to (\mu^{*}_{i})^{-\gamma}\le c^{-\gamma}<\infty$ as $t\to\infty$.
    Next, fix $i\notin \mathcal{A}^{*}_{\text{global}}$. By definition, $\mu_{i}^{*}=0$ and thus $\tilde{w}_{i}^{(t)}\to\infty$. Note that these separations do not use Assumption~\ref{assum:neglibiblebiasterm}.
    
    \textbf{Step 2: Convergence of the AEN-SAE objective and shrinkage elimination.}
    The weight separations together with Assumption \ref{assum:neglibiblebiasterm} give the pointwise convergence of the AEN-SAE objective:
    \begin{align}
    \begin{split}
    \label{eq: pointwise convergence to the oracle}
        &\frac{1}{2}\|x-D h(x)\|_{2}^{2}
        +\lambda_{1} \sum_{i=1}^{d_{\text{dict}}}\tilde{w}^{(t)}_{i}h_{i}(x)+\lambda_{2}\|h(x)\|_{2}^{2}\\
        &\to
        \frac{1}{2}\|x-D h(x)\|_{2}^{2}
        +\lambda_{2}\|h(x)\|_{2}^{2}
        + 1_{h_{i}(x)=0, i\notin\mathcal{A}^{*}_{\text{global}}}.
        \end{split}
    \end{align}
    The first term of the AEN-SAE objective gives a uniform bound on the minimizers, so the minimization can be restricted to a compact subset of the feasible region.  Since the limiting objective that defines $h^{*}$ is strongly convex and therefore has a unique minimizer, the minimizers satisfy $\hat{h}^{(t)}(x)\to h^{*}(x)$, which shows the shrinkage elimination. As $\mathbb{E}_{x}[\|x\|_{2}^{2}]<\infty$, we also have $\mathbb{E}_{x}[\hat{h}^{(t)}(x)\mid \tilde{w}^{(t)}]\to \mu^{*}$.

    \textbf{Step 3: Global oracle selection.}
    Fix $i\in\mathcal{A}^{*}_{\text{global}}$.
    The convergence $\mathbb{E}_{x}[\hat{h}^{(t)}_{i}(x)\mid \tilde{w}^{(t)}]\to \mu_{i}^{*}$,
    together with the signal lower bound (Assumption \ref{assum:signal}), implies that taking a sufficiently large $t$ gives
    $\mathbb{E}_{x}[\hat{h}^{(t)}_{i}(x)\mid \tilde{w}^{(t)}]>\tau$.
    Meanwhile, fix $i\notin\mathcal{A}^{*}_{\text{global}}$.
    The convergence $\mathbb{E}_{x}[\hat{h}^{(t)}_{i}(x)\mid \tilde{w}^{(t)}]\to 0$ implies that taking a sufficiently large $t$ gives
    $\mathbb{E}_{x}[\hat{h}^{(t)}_{i}(x)\mid \tilde{w}^{(t)}]<\tau$.
    Therefore, we get \eqref{eq: timescale identification}, which completes the proof.
\end{proof}

The oracle property has implications beyond feature selection; it also improves the geometric stability of the target manifold. By allowing noise features to receive an infinitely diverging penalty, the adaptive weights effectively decouple the active signal from ambient dictionary noise, naturally suppressing the interaction leakage term of the Hoffman bound.

\begin{corollary}[Suppression of Interaction Leakage]
\label{cor:leakage}
Let $\hat{h}$ be the AEN-SAE code given the weights $w_{i}$s.
For active features $\mathcal{A}=\{i:\hat{h}_{i}\ne 0\}$,
let $\Delta h_{\mathcal{A}}$ be a local perturbation along the active manifold. If the minimum weight for inactive features satisfies $\lambda_1 \min_{i\in\mathcal{A}^{c}}w_{i} > \|D_{\mathcal{A}^c}^\top (D_{\mathcal{A}} \Delta h_{\mathcal{A}}-x)\|_\infty$, then the strict subdifferential inclusion for the inactive orthant is preserved. Consequently, the local Hoffman bound $H$ is independent of the interaction term $\|D_{\mathcal{A}^c}^\top D_{\mathcal{A}}\|_2$.
\end{corollary}

\begin{proof}
We begin with the first order optimality condition to the AEN-SAE objective:
\begin{align}
-D_i^\top (x - D\hat{h}) + 2\lambda_2 \hat{h}_i + \lambda_1 w_i \cdot \text{sign}(\hat{h}_i) = 0 \quad &\text{for } \hat{h}_i \neq 0, \label{eq:kkt1_ideal} \\
|D_i^\top (D\hat{h} - x)| \le \lambda_1 w_i \quad &\text{for } \hat{h}_i = 0. \label{eq:kkt2_ideal}
\end{align}

A local perturbation $\Delta h_{\mathcal{A}}$ in the active set shifts the gradient residual for the inactive set by $D_{\mathcal{A}^c}^\top (D_{\mathcal{A}} \Delta h_{\mathcal{A}}-x)$. For the polyhedral topology to remain stable (i.e., no inactive feature improperly enters the support), the first order optimality condition for the inactive orthant \eqref{eq:kkt2_ideal} 
must continue to hold under this perturbation. This requires the subgradient barrier to absorb the induced gradient leakage:
\begin{equation*}
    \|D_{\mathcal{A}^c}^\top (D_{\mathcal{A}} \Delta h_{\mathcal{A}}-x)\|_\infty \le \lambda_1 w_{i \in \mathcal{A}^c}.
\end{equation*}
Under this condition, the inactive variables $h_{\mathcal{A}^c}$ are thus clamped to zero. Therefore the local polyhedral stability is decoupled from the interaction leakage $\|D_{\mathcal{A}^c}^\top D_{\mathcal{A}}\|_2$, and is instead controlled by the active curvature established in Theorem~\ref{thm:stability}, which completes the proof.
\end{proof}

Note that in a standard $\ell_1$ SAE, the subgradient barrier for inactive features is uniformly $\lambda_1$, meaning any cross-correlation leakage exceeding $\lambda_1$ causes a discontinuous boundary shift, driving $H \to \infty$. In the AEN-SAE, however, the clipped adaptive penalty $\text{clip}((\bar{h}^{(t)}_{i})^{-\gamma},w_{\min},w_{\max})$ for inactive features diverges toward $w_{\max}$; see the proof of Theorem \ref{thm:oracle}. So choosing $\lambda_{1}w_{\max}$ sufficiently large guarantees the stated condition.

Corollary~\ref{cor:leakage} completes our analysis of the Hoffman bound in Section~\ref{subsec:amortized-lasso}. While the overcomplete dictionary $D$ is inherently highly collinear (which would normally cause the numerator $\|D_{\mathcal{A}^c}^\top D_{\mathcal{A}}\|_2$ to diverge), the AEN-SAE's adaptive weights neutralize its impact. Combined with Theorem~\ref{thm:stability}, which bounds the denominator via $\lambda_2$, the AEN-SAE aims to solve the root causes of polyhedral optimization pathologies.

\subsection{Numerical Stability and Implementation Details}
\label{subsec:stability-mechanisms}

While Theorems \ref{thm:lipschitz} and \ref{thm:oracle} establish the stability and asymptotic consistency of the idealized AEN-SAE, deploying this algorithm at the scale of LLMs ($d_{\text{dict}} > 10^5$) introduces several numerical and optimization constraints. To translate the idealized mathematical oracle into the robust, highly scalable implementation used in our experiments, we introduce three necessary engineering modifications:

\paragraph{1. Scale invariance from reference features.} The idealized weight $\tilde{w}_i = (\bar{h}_i)^{-\gamma}$ is highly sensitive to the raw magnitude of the LLM residual stream, which varies drastically across different models and intermediate layers. To make the architecture transferable without exhaustive hyperparameter retuning, we introduce a scale-invariant reference $\bar{h}_{\text{ref}}^{(t)}$, defined as the mean activity of the top $K\%$ features. The penalty is calculated relative to this active cohort: $w_i^{(t)} \propto (\bar{h}_{\text{ref}}^{(t)} / \bar{h}_i^{(t)})^\gamma$. This ensures the relative penalty distribution remains identical regardless of whether the layer's activations are bounded by $10^0$ or $10^2$.

\paragraph{2. Numerical bounds ($w_{\min}$ and $w_{\max}$).} In theory, we rely on the idealized limit $\tilde{w}_i \to \infty$. In 16-bit (BF16) or 32-bit floating-point hardware, this unbounded growth causes numerical instability and gradient explosion. We therefore clamp the maximum penalty to $w_{\max}$. As long as $\lambda_1 w_{\max}$ is sufficiently larger than the local gradient residual, the KKT inequality holds, and the noise feature remains safely dead. Conversely, we clamp the minimum penalty to $w_{\min}$ (e.g., $0.1$) to prevent underflow, recognizing that in practice, shrinkage is not strictly zero, but safely bounded to a negligible $\lambda_1 w_{\min}$.

\paragraph{3. The cold-start warmup schedule.} The asymptotic proof naturally assumes the system has already reached the steady-state expectation $\mathbb{E}[|h_i^*|]$. However, during the initial steps of neural network training, feature activations are driven by random initialization noise. If an ultimately useful feature initializes poorly, its EMA drops, its weight spikes to $w_{\max}$, and it is permanently suppressed before the encoder has the opportunity to route data to it. We mitigate this problem by interpolating the adaptive weights with a delayed linear warmup schedule $\rho(t) \in [0, 1]$. This effectively enforces a uniform $\ell_1$ penalty until the dictionary has sufficiently aligned with the data manifold, satisfying the timescale separation requirement (Assumption~\ref{assum:timescale}) before the adaptive penalties take effect.

\section{Experimental Setup and Full Training Details}
\label{app:setup}

\paragraph{Data pipeline and activation extraction.} For LLM-based experiments, we stream text from a deduplicated version of the Pile~\citep{gao2020pile800gbdatasetdiverse}. Tokens are packed contiguously into fixed-length sequences without padding to maximize throughput. Residual-stream activations are extracted from a frozen LLM via a forward hook at a chosen layer, and flattened into token-level batches.

To ensure consistent scaling across models, each activation vector $x \in \mathbb{R}^{d_{\text{model}}}$ is $\ell_2$-normalized per token by rescaling to $\|x\|_2 = \sqrt{d_{\text{model}}}$, yielding approximately unit variance per coordinate.

A circular buffer is used to accumulate activations into fixed-size SAE minibatches, improving hardware utilization and throughput.

\paragraph{Training Procedure.} At each training step, the SAE reconstructs activations and optimizes the reconstruction loss with regularization using Adam at a fixed learning rate. Gradients are clipped to prevent instability and decoder columns are renormalized to unit $\ell_2$ norm after each update.

Encoder and decoder weights are initialized using Kaiming uniform initialization. The decoder bias is initialized to zero and the decoder dictionary is explicitly normalized both at initialization and after each optimization step.

For adaptive-weight models, EMA statistics are updated at every step, and the adaptive reweighting mechanism is activated only after a warmup period. This ensures stable initialization before adaptive penalties are applied.

\paragraph{Baselines and compute normalization.} We include a TopK baseline without auxiliary resampling or proxy-gradient mechanisms. This serves two purposes: (i) to normalize computational cost across methods, and (ii) to isolate whether AEN-SAEs resolve feature starvation without relying on heuristic interventions. No auxiliary resampling or proxy-gradient methods are used for any model, including AEN-SAEs.

\paragraph{Metrics and evaluation.} We log reconstruction and sparsity metrics, including mean squared error (MSE), explained variance, $\ell_0$ sparsity, and the $\ell_1/\ell_0$ ratio. Feature starvation is quantified via dead neuron statistics, including windowed dead-feature percentage and recovery rates. We additionally track shrinkage ratios and feature-utilization summaries.

To evaluate geometric structure, we compute dictionary diagnostics such as mutual coherence, condition numbers of active Gram matrices, and interaction leakage measures. Feature-utilization metrics include nromalized entropy, Gini coefficient, and mass concentration across feature subsets.

For LLM-based experiments, we assess downstream fidelity by patching SAE reconstructions into the frozen model and measuring cross-entropy (CE) degradation, CE recovery relative to a batch-mean ablation baseline, and KL divergence between output distributions.

\paragraph{Validation protocol.} Validation is performed in two regimes. First, an online validation stream loops over a held-out dataset shard and evaluates a subset of documents at each logging step to track training dynamics. Second, a final validation pass is conducted over a non-looping held-out shard, which is exhausted once to compute aggregate metrics reported in tables.

\section{Extended Discussion of Metrics}
\label{sec:extended-methodology}

In this appendix, we describe each metrics logged in detail. Some of these are inspired by classical high-dimensional statistics and may offer novel ways to measure SAE health. Throughout, let $x \in \mathbb{R}^d$ denote an input activation, $\hat{x}$ its SAE reconstruction, $h \in \mathbb{R}^{d_{\text{dict}}}$ the latent code, and  $D \in \mathbb{R}^{d_{\text{model}} \times d_{\text{dict}}}$ the decoder dictionary with columns $d_j$. Let $\mathcal{A}$ denote an active feature set, $\mathcal{A}^c$ its complement, and $\epsilon > 0$ a small numerical constant.

For all scalar metrics, we report the mean, standard deviation, and the 10th, 50th, and 90th percentiles. Where relevant, we additionally track the minimum and maximum values.

\subsection{Reconstruction and Sparsity}

The most natural ways of measuring reconstruction performance are the mean-squared error and explained variance
\begin{align*}
    \text{MSE} &= \mathbb{E}\!\left[\|x - \hat{x}\|_2^2\right], \\
    \text{Explained variance} &= 1 - \frac{\sum_{j=1}^{d_{\text{model}}}\mathrm{Var}(x_j - \hat{x}_j)}{\sum_{j=1}^{d_{\text{model}}}\mathrm{Var}(x_j)}.
\end{align*}

For an SAE, the goal is for the internal representation to be sparse. Hence, we measure the average sparsity by the number of active features:
\begin{equation*}
    \text{$\ell_0$ active features} = \mathbb{E}\!\left[\sum_{i=1}^{d_{\text{dict}}}\mathbf{1}\{|h_i| > \epsilon\}\right].
\end{equation*}
The effect of shrinkage bias can be measured with the ratio of reconstructed to original $\ell_1$ magnitude, or the $\ell_1$ magnitude per active feature. Respectively:
\begin{align*}
    \text{Feature shrinkage ratio} &= \frac{\mathbb{E}\!\left[\|\hat{x}\|_1\right]}{\mathbb{E}\!\left[\|x\|_1\right]}, \\
    \text{$\ell_1/\ell_0$ ratio} &= \mathbb{E}\!\left[\frac{\|h\|_1}{\max(\|h\|_0,1)}\right],
\end{align*}
where a feature shrinkage ratio of 1 means the magnitude of the original signal has been preserved.

\subsection{Dead Neuron Metrics}

A central focus of our study is on the existence and recovery of dead neurons. Let $m_i$ be the maximum absolute activation of feature $i$ over a fixed window of training steps. We define the dead neuron percentage as:
\begin{equation*}
    \text{Dead neurons pct} = 100 \cdot \frac{1}{d_{\text{dict}}} \sum_{i=1}^{d_{\text{dict}}}\mathbf{1}\{m_i \le \epsilon\}.
\end{equation*}
Given two consecutive windows with maxima $m_i^{\text{prev}}$ and $m_i^{\text{curr}}$, the recovery rate is:
\begin{equation*}
    \text{Dead neuron recovery rate} = \frac{\sum_i \mathbf{1}\{m_i^{\text{prev}} \le \epsilon\}\mathbf{1}\{m_i^{\text{curr}} > \epsilon\}}{\sum_i \mathbf{1}\{m_i^{\text{prev}} \le \epsilon\}},
\end{equation*}
with the convention that the value is 0 when the denominator is 0. The window size is set to 10,000 steps in our experiments. This leads to between ten million and fifty million tokens seen per window, allowing rare features to fire.

For LLM batches, we additionally log the percentage of features that never activate within the batch---though this tends to be a noisy metrics as rare semantic features are often marked as dead.
\begin{equation*}
    \text{Dead neurons batch pct} = 100 \cdot \frac{1}{d_{\text{dict}}} \sum_{i=1}^{d_{\text{dict}}}\mathbf{1}\!\left\{\max_b |h_{b,i}| \le \epsilon\right\}.
\end{equation*}

\subsection{Adaptive-Weight Metrics}

For adaptive LASSO and adaptive elastic net SAEs, measuring the concentration of the weights is useful to observe degeneracy. Using the same notation as in Section~\ref{sec:aen}, the activation effective sample size (ESS) is:
\begin{equation*}
    \text{ESS} = \frac{\left(\sum_i a_i\right)^2}{\sum_i a_i^2}, \qquad a_i = \max(\bar{h}_i, 0).
\end{equation*}
We also report summary statistics of the adaptive weights, together with the fractions pinned near the bounds, which can be used to adjust the weight clipping in the case of degeneracy:
\begin{align*}
    \text{Pinned max weight pct} &= 100 \cdot \frac{1}{d_{\text{dict}}}\sum_i \mathbf{1}\{w_i \le w_{\min} + \varepsilon\}, \\
    \text{Pinned min weight pct} &= 100 \cdot \frac{1}{d_{\text{dict}}}\sum_i \mathbf{1}\{w_i \ge w_{\max} - \varepsilon\},
\end{align*}

\subsection{Feature-Utilization Metrics}

We argue that the distribution of firing features is a proxy for the health of an SAE. Let $r_i$ be the per-feature firing rate, estimated either from the EMA firing rate or from batch frequencies. The normalized entropy is:
\begin{equation*}
    \text{Normalized entropy} = \frac{-\sum_i p_i \log p_i}{\log d_{\text{dict}}}, \qquad p_i = \frac{r_i}{\sum_j r_j}.
\end{equation*}
The concentration of features firing can also be measured. The top-$n$\% feature-firing mass is
\begin{equation*}
    \text{Top-$n$\% mass pct} = 100 \cdot \sum_{i \in \mathrm{top}\,1\%} p_i.
\end{equation*}
The Gini coefficent, as a measure of computing inequality in firing rates, is computed from the sorted rates $r_{(1)} \le \cdots \le r_{(n)}$ as:
\begin{equation*}
    \text{Gini} = \frac{2\sum_{i=1}^n i\,r_{(i)}}{n\sum_i r_i} - \frac{n+1}{n}.
\end{equation*}

\subsection{Dictionary Geometry Metrics}

Let $\mathcal{A}$ be the active set for a given batch, with decoder blocks $D_{\mathcal{A}}$ and $D_{\mathcal{A}^c}$. Since the Frobenius norm upper bounds the spectral norm ($\| D\|_2 \le \| D\|_F$), a proxy for the interaction leakage, which is one of the factors controlling the stability of LASSO solution maps, can be measured as 
\begin{equation*}
    \text{Frobenius interaction} = \|D_{\mathcal{A}^c}^\top D_{\mathcal{A}}\|_F.
\end{equation*}
The other factor in stability is the curvature of the active Gram matrix is $G_{\mathcal{A}} = D_{\mathcal{A}}^\top D_{\mathcal{A}}$, from which we report:
\begin{align*}
    \text{Active min eigenvalue} &= \lambda_{\min}(G_{\mathcal{A}}),\\ \text{Active max eigenvalue} &= \lambda_{\max}(G_{\mathcal{A}}),\\ \text{Active condition number} &= \frac{\lambda_{\max}(G_{\mathcal{A}})}{\max(\lambda_{\min}(G_{\mathcal{A}}),\epsilon)}.
\end{align*}
Although metrics for the full Gram matrix $G = D^\top D$ will be intractable for large models, they are not prohibitive in the active space since it aims to be sparse with $|\mathcal{A}|$ on the scale of a few hundred features at most.

As demonstrated, dictionary coherence can affect the choice of appropriate SAE architecture greatly, as predicted in the classical theory of LASSO~\citep{wainwright2019highdimensional}. Let $\tilde d_j = d_j / \|d_j\|_2$. The nearest-neighbour coherence of feature $j$ is $c_j = \max_{k \ne j} \left|\tilde d_j^\top \tilde d_k\right|$. We report:
\begin{align*}
    \text{Max mean cosine similarity} &= \frac{1}{d_{\text{dict}}}\sum_j c_j,\\ \text{Dictionary coherence nearest-neighbour max} &= \max_j c_j,\\ \text{Dictionary coherence max} &= \max_{j \ne k} \left|\tilde d_j^\top \tilde d_k\right|,\\ \text{Dictionary coherence mean} &= \frac{1}{d_{\text{dict}}(d_{\text{dict}}-1)} \sum_{j \ne k}\left|\tilde d_j^\top \tilde d_k\right|.
\end{align*}

\subsection{LLM Downstream-Fidelity Metrics}

For validation on frozen language models, we evaluate next-token cross-entropy under three forward-pass conditions: the original uninterrupted model (\textit{clean}), the model with SAE reconstructions patched into the residual stream (\textit{patched}), and an ablated baseline (\textit{baseline}).

Let $z$ denote logits and $t$ denote tokens. The next-token cross-entropy is:
\begin{equation*}
    \mathrm{CE}(z,t) = -\frac{1}{N}\sum_{n=1}^{N}\log p_z(t_{n+1}\mid t_{\le n}).
\end{equation*}

We use a batch-mean ablation for our baseline logits ($z_{\mathrm{baseline}}$) rather than standard zero-ablation. As noted in recent mechanistic interpretability literature, zero-ablating intermediate residual streams in modern architectures pushes the network off the valid activation manifold~\citep{nanda2022glossary}. Specifically, substituting zero vectors can destabilize rotary position embeddings (RoPE)~\citep{su2023rope} and induce uncalibrated out-of-distribution shifts that artificially inflate the apparent impact of the ablation~\citep{pochinkov2024investigating}. Mean ablation mitigates this geometric instability by keeping the baseline interventions closer to the true data distribution~\citep{li2024optimal}. While it constitutes a strictly stronger, more challenging baseline---as the mean activation inherently reconstructs high-frequency, common features and lacks only fine-grained semantic nuance---it provides a much cleaner, causally rigorous isolation of the SAE's ability to recover token-specific representations.

The batch-level metrics are:
\begin{align*}
    \text{Clean CE} &= \mathrm{CE}(z_{\mathrm{clean}}, t),\\ 
    \text{Patched CE} &= \mathrm{CE}(z_{\mathrm{patched}}, t),\\ 
    \text{Baseline CE} &= \mathrm{CE}(z_{\mathrm{baseline}}, t),\\ \text{CE degradation} &= \mathrm{patched\_loss} - \mathrm{clean\_loss},\\ \text{CE recovered} &= 1 - \frac{\text{Patched CE} - \text{Clean CE}}{\text{Baseline CE} - \text{Clean CE}},\\ 
\end{align*}
We also compute the KL divergence between clean and patched logits:
\begin{equation*}
   \mathrm{KL}(p_{\mathrm{clean}} \,\|\, p_{\mathrm{patched}}) = \sum_v p_{\mathrm{clean}}(v) \log\frac{p_{\mathrm{clean}}(v)}{p_{\mathrm{patched}}(v)}.
\end{equation*}
In implementation, logits are sanitized to finite values before computing the softmax and KL. A set of final test metrics are reported for all metrics on a corpus of held-out documents.

\subsection{Training-Health and Throughput Metrics}

Finally, we log gradient norms and update ratios for the global network, as well as decomposed at the encoder/decoder level. This is used for validating the timescale separation assumption.
\begin{align*}
    \text{Global grad norm} &= \|\nabla\theta\|_2, \\
    \text{Encoder grad norm} &= \|\nabla\theta_{\text{enc}}\|_2, \\ 
    \text{Decoder grad norm} &= \|\nabla\theta_{\text{dec}}\|_2,
\end{align*}
The normalized update magnitudes are thus:
\begin{align*}
    \text{Encoder update ratio} &= \eta \frac{\|\nabla \theta_{\mathrm{enc}}\|_2}{\|\theta_{\mathrm{enc}}\|_2}, \\
    \text{Decoder update ratio} &= \eta \frac{\|\nabla \theta_{\mathrm{dec}}\|_2}{\|\theta_{\mathrm{dec}}\|_2}.
\end{align*}
We also profile FLOPs per step and cumulatively in order to compare the computational cost of different architectures.

\section{Extended Spike Model Analysis}
\label{app:spiked}

\subsection{Data Generation and Training Protocol}

To study the effect of dictionary coherence in a controlled setting, we consider a spiked covariance model following~\citet{wainwright2019highdimensional}. Each dictionary atom $d_j \in \mathbb{R}^{d_{\text{model}}}$ is constructed as a mixture of an independent component $u_j$ and a shared spike $v$:
\begin{equation*}
    d_j = \frac{\sqrt{1-\rho}\,u_j + \sqrt{\rho}\,v}{\|\sqrt{1-\rho}\,u_j + \sqrt{\rho}\,v\|_2},
\end{equation*}
where $u_j \sim \mathcal{N}(0, I)$ are independent directions and $v \sim \mathcal{N}(0, I)$ is shared across all atoms. The parameter $\rho \in [0,1]$ controls the expected pairwise coherence between atoms, interpolating between independent dictionaries ($\rho=0$) and highly collinear ones ($\rho \to 1$).

For each coherence level $\rho \in \{0, 0.3, 0.5, 0.7, 0.9\}$, we construct a fixed teacher dictionary $D^* \in \mathbb{R}^{d_{\text{model}} \times d_{\text{dict}}}$ with $d_{\text{model}}=256$ and $d_{\text{dict}}=1024$ (a $4\times$ overcomplete representation).

Synthetic data is generated by sampling sparse codes $h^* \in \mathbb{R}^{d_{\text{dict}}}$ with exactly $k=16$ nonzero entries. Active indices are drawn uniformly without replacement, and nonzero magnitudes are sampled independently from $\mathrm{Uniform}[1,3]$. Observations are then formed as $x^* = D^* h^*.$

All SAE variants are trained on the synthetic data using Adam for $50{,}000$ steps with batch size $256$ and learning rate $10^{-3}$. Hyperparameters are swept to target specific sparsity levels, and performance is evaluated in terms of reconstruction accuracy, sparsity, and feature utilization.

\subsection{Low-Coherence Regime}

When $\rho=0$, dictionary atoms are approximately orthogonal, and the sparse recovery problem is well-conditioned. In this regime, all architectures perform comparably in terms of reconstruction, and feature starvation is largely absent for continuous relaxation methods.

Hard-masking methods such as TopK perform well because feature selection is trivial in the absence of correlation. AEN-SAE maintains full dictionary utilization (no dead features) while incurring a modest reduction in explained variance due to the $\ell_2$ regularization term.

\subsection{High-Coherence Regime}

As coherence increases, the geometry of the sparse recovery problem deteriorates. At $\rho=0.9$, atoms become highly collinear, violating the irrepresentability condition and inducing instability in $\ell_1$-based methods.

\paragraph{Continuous Relaxations.} Standard $\ell_1$-regularized SAEs exhibit strong sensitivity to hyperparameters, oscillating between dense, high-shrinkage solutions and degenerate sparse representations. Adaptive LASSO, while theoretically appealing, becomes numerically unstable in this regime; in our experiments, the majority of hyperparameter configurations collapse to degenerate solutions ($\ell_0 < 2$), which is consistent with known limitations of $\ell_1$ methods under high collinearity.

\paragraph{Hard-Masking Methods.} TopK exhibits extremely severe feature starvation in the high-coherence regime. Because feature selection is based on raw activation magnitude, highly correlated features compete, which leads to repeated selection of a small subset of dominant atoms. As a result, over $95\%$ of dictionary elements become permanently inactive across a range of sparsity levels.

\paragraph{AEN-SAE Behavior.} AEN-SAE mitigates the effects of hard-masking methods by combining adaptive reweighting with an $\ell_2$ regularization term. The $\ell_2$ term stabilizes the curvature of the optimization landscape, while adaptive weights suppress redundant features and encourage diversification of the active set.

Empirically, AEN reduces dead features to approximately $40\%$ at typical sparsity levels. While feature starvation is not completely eliminated, the reduction is substantial compared to hard-masking methods. Additionally, AEN achieves improved geometric conditioning, reducing the active-set condition number (e.g., $\kappa \approx 21$ compared to $\kappa \approx 769$ for standard elastic net) and lowering interaction leakage between active and inactive features.

\subsection{Discussion}

Our results highlight a fundamental limitation of $\ell_1$-based sparse coding under coherence: as dictionary atoms become increasingly correlated, the associated solution map becomes unstable and difficult to approximate with shallow encoders. Even in this controlled setting, feature starvation remains significant, suggesting that it is not purely an optimization artifact but a structural consequence of coherence.

AEN-SAE provides a principled mechanism for mitigating these effects by directly controlling the geometric factors governing stability. While it does not fully eliminate feature starvation, it substantially improves feature utilization and robustness without relying on auxiliary heuristics such as resampling or proxy gradients.

\section{Extended Analysis on Pythia 70M}
\label{app:pythia}

\subsection{Experimental Setup}

We evaluate SAE variants on activations extracted from layer 3 of a frozen Pythia 70M model~\citep{biderman2023pythia}. The residual stream has dimension $d_{\text{model}} = 512$ and we train an $8\times$ overcomplete dictionary with $d_{\text{dict}} = 4096$.

Training data is drawn from a deduplicated version of the Pile~\citep{gao2020pile800gbdatasetdiverse}, with sequence length 128 and per-token $\ell_2$ normalization applied to activations. Optimization proceeds using Adam for 25,000 steps (approximately $204.8$M tokens), with batch size 8,192 and learning rate $10^{-3}$. AEN-SAEs employ a 4,000-step warmup for adaptive weights. Validation is performed on a held-out shard of 5,000 documents.

\subsection{Hyperparameter Tuning and Timescale Separation}

AEN-SAEs introduce several hyperparameters, including the $\ell_2$ regularization strength $\lambda_2$, the adaptive weighting exponent $\gamma$, the EMA momentum $\beta$, weight bounds $[w_{\min}, w_{\max}]$, and the reference percentile $p$. A full hyperparameter sweep is provided in the next section, Appendix~\ref{sec:hyperparam-sweep-pythia}.

In practice, these parameters exhibit strong stability across runs. Following classical adaptive elastic net protocol, we set $\gamma \in \{0.5, 1.0\}$ and choose $\lambda_2 \leq \lambda_1$, typically setting it to a small value and treating it as a perturbation for structural $\ell_2$ regularization. Adaptive weights are clipped to $[0.01, 10.0]$ for numerical stability.

A key theoretical requirement is timescale separation between the EMA buffer and dictionary updates. With $\beta \in \{0.999, 0.9999\}$, the EMA update rate $(1 - \beta)$ is $10^{-3}$--$10^{-4}$, corresponding to a memory horizon spanning tens of millions of tokens. Empirically, we observe that parameter updates occur at magnitude $\mathcal{O}(10^{-7})$ per step, which implies that the activation statistics evolve several orders of magnitude faster than the dictionary. This satisfies the timescale separation assumption required for stable adaptive weighting.

To initialize the EMA buffer, we couple the warmup duration to the EMA horizon via $T_{\text{warmup}} \approx (1 - \beta)^{-1}$, ensuring that adaptive weights are only applied after sufficient signal accumulation. Finally, because our adaptive weights are scaled relative to the top $p=0.05$ percentile of features, the entire AEN-SAE architecture is naturally scale-invariant; good hyperparameter choices should generalize to large models. Overall, AEN-SAEs introduce zero meaningful computational overhead---measuring less than a 0.1\% increase in FLOPs per step compared to a purely TopK implementation---and only require tuning the base penalty $\lambda_1$ to hit target sparsity bottlenecks just as in vanilla SAEs.

\subsection{Hyperparameter Sweep for Pythia 70M}
\label{sec:hyperparam-sweep-pythia}

For a robust comparison, we conducted hyperparameter sweeps for all SAE architectures. The complete grid of hyperparameters evaluated in our Pythia 70M experiments is detailed in Table~\ref{tab:pythia-hyperparameter-sweep}.

For our adaptive architectures (adaptive LASSO and AEN-SAE), several structural hyperparameters were held constant across runs to isolate the effect of the sparsity penalty. The bounds on the adaptive weights were fixed to $[w_{\min}, w_{\max}] = [0.01, 10.0]$ to ensure numerical stability without causing gradient explosion. To prevent premature feature death during the noisy initial phase of optimization, the adaptive weighting mechanism utilized a delayed warmup schedule: the effective penalty was held strictly uniform ($\rho(t) = 0$) for the first 4,000 training steps, followed by a linear interpolation to full adaptive weighting ($\rho(t) = 1$) over the subsequent 2,000 steps.

\begin{table}[htbp]
\centering
\scriptsize
\caption{\textbf{Pythia 70M hyperparameter sweep grid.} The combinatorial sweep of hyperparameter configurations evaluated for each sparse autoencoder architecture.}
\label{tab:pythia-hyperparameter-sweep}
\begin{tabular}{ll}
\toprule
\textbf{Architecture} & \textbf{Hyperparameter Grid} \\
\midrule
\textbf{Top-K} & $K \in \{8, 16, 32, 64, 128, 256\}$ \\
\midrule
\textbf{Vanilla $\ell_1$} & $\lambda_1 \in \{0.0005, 0.001, 0.005, 0.01, 0.05\}$ \\
\midrule
\textbf{Elastic Net} & $\lambda_1 \in \{0.001, 0.005, 0.01\}$ \\
& $\lambda_2 \in \{0.0001, 0.0005, 0.001, 0.005\}$ \\
\midrule
\textbf{Adaptive LASSO} & $\lambda_1 \in \{0.001, 0.005, 0.01\}$ \\
& $\gamma \in \{0.5, 1.0\}$ \\
& $\beta \in \{0.999, 0.9999\}$ \\
& $p \in \{0.005, 0.01, 0.05\}$ \\
\midrule
\textbf{AEN-SAE} & $\lambda_1 \in \{0.001, 0.005, 0.01\}$ \\
& $\lambda_2 \in \{0.0001, 0.0005, 0.001, 0.005\}$ \\
& $\gamma \in \{0.5, 1.0\}$ \\
& $\beta \in \{0.999, 0.9999\}$ \\
& $p \in \{0.005, 0.01, 0.05\}$ \\
\bottomrule
\end{tabular}
\end{table}

\subsection{Coherence-Induced Failure Modes}

Unlike synthetic sparse coding problems, LLM activations exhibit strong contextual correlations, inducing a highly coherent dictionary as characterized by high mutual cosine similarity between atoms~\citep{zou2006adaptivelasso, fan2008sureindependencescreeningultrahigh}. This violates standard sparse recovery assumptions such as the irrepresentability condition and leads to predictable failure modes across SAE architectures.

\paragraph{Vanilla $\ell_1$.} To suppress correlated noise features, the $\ell_1$ penalty must be increased, which induces significant shrinkage bias. At $\ell_0 \approx 22$, we observe magnitude suppression of approximately $25\%$, which degrades both reconstruction fidelity and downstream performance. The method is also highly sensitive to hyperparameter choice, reflecting the instability of $\ell_1$ regularization under collinearity.

\paragraph{Elastic net.} The addition of an $\ell_2$ penalty stabilizes optimization but leads to dense and poorly conditioned solutions. In our experiments, elastic net produces representations with $\ell_0 > 200$ and extremely high condition numbers (median $\kappa > 10^{13}$), indicating that it fails to meaningfully separate features in coherent regimes.

\paragraph{Adaptive LASSO.} Although adaptive LASSO benefits from strong theoretical guarantees in low-coherence settings, it becomes unstable in this regime due to the absence of a structural curvature term. Across our sweeps, a large fraction of configurations collapse to degenerate solutions, failing to capture meaningful structure. This behavior is consistent with known limitations of $\ell_1$-based methods under strong feature correlations \cite{zou2009adaptiveelasticnet}; see also Theorem \ref{thm:lipschitz} and the subsequent discussion.

\subsection{Reconstruction--Utilization Trade-off}

Figure~\ref{fig:pythia70m-pareto} and Table~\ref{tab:pythia_aen_vs_topk} demonstrate a consistent trade-off between reconstruction accuracy and feature utilization.

TopK achieves higher explained variance across sparsity levels (e.g., $0.76$ at $\ell_0 \approx 32$), but it does so by concentrating activation mass into a small number of highly redundant features. This phenomenon is reflected in high maximum coherence ($\approx 0.99$) and moderate levels of feature starvation.

In contrast, AEN-SAEs distribute activation mass across correlated feature groups. While this leads to a modest reduction in explained variance, it significantly improves feature utilization, reduces redundancy, and produces a more balanced dictionary.

\subsection{Feature Utilization and Activation Concentration}

At relaxed sparsity levels ($\ell_0 \approx 128$), the differences between methods become more evident. TopK concentrates approximately $86\%$ of activation mass in the top $10\%$ of features, yielding a high Gini coefficient ($0.903$) and persistent feature starvation. AEN-SAEs mitigate this concentration, reducing the Gini coefficient to $0.670$ and increasing entropy, while lowering dead features to approximately $3\%$. This reflects a more uniform utilization of the dictionary.

\subsection{Mechanism: Adaptive Reweighting}

The improved utilization arises from the adaptive weighting mechanism. Frequently activated features accumulate larger EMA values, increasing their effective penalty and preventing them from dominating the representation. Conversely, underutilized features receive relatively lower penalties, allowing them to become active. This induces a self-balancing effect that redistributes activation mass away from redundant hubs and toward underutilized features, which improves coverage of the representation space without requiring explicit resampling or auxiliary losses.

\subsection{Discussion}

Overall, our results demonstrate that coherence is a primary driver of SAE failure in real-world settings. Hard-masking methods achieve strong reconstruction performance by collapsing onto a small subset of features, while AEN-SAEs provide a principled mechanism for improving feature utilization and geometric stability with minimal computational overhead and no auxiliary heuristics.





\end{document}